\documentclass[accepted]{uai2026}

\usepackage[utf8]{inputenc}
\usepackage[T1]{fontenc}
\usepackage{microtype}
\usepackage[english]{babel}

\usepackage{natbib}

\usepackage{authblk}

\usepackage{algorithm}
\usepackage{algpseudocode}

\usepackage{graphicx}
\usepackage{subcaption}
\usepackage{xcolor}
\usepackage{tikz}
\usetikzlibrary{arrows.meta, positioning, calc}
\usepackage{pgfplots}
\pgfplotsset{compat=1.18}
\usepackage{booktabs}

\usepackage{multirow} 
\usepackage{makecell}
\usepackage{adjustbox}

\usepackage{amsmath}
\usepackage{amssymb}
\usepackage{mathtools}
\usepackage{amsthm}
\usepackage{amsfonts}
\usepackage{bm}

\usepackage{enumitem}
\setlist{nosep}
\usepackage{lineno}

\usepackage{quantikz}
\theoremstyle{remark}
\usepackage{placeins}      
\usepackage{caption}       
\usepackage{dblfloatfix}   
\usepackage{bm}        
\usepackage{braket}    

\usepackage[most]{tcolorbox}
\usepackage{fontawesome5}

\definecolor{takeawaybg}{HTML}{F2F7F7}
\definecolor{takeawaystroke}{HTML}{468985}
\definecolor{theorbg}{gray}{0.96}
\definecolor{theorframe}{gray}{0.40}

\newtcolorbox{contribbox}{
  enhanced,
  colback=green!5,
  colframe=green!40!black,
  boxrule=0pt,
  arc=0mm,
  left=2mm, right=2mm, top=1.5mm, bottom=1.5mm,
  borderline west={3pt}{0pt}{green!60!black},
}

\usepackage{hyperref}

\definecolor{uaiblue}{RGB}{0, 0, 120}

\hypersetup{
  colorlinks=true,
  linkcolor=uaiblue,
  citecolor=uaiblue,
  urlcolor=uaiblue
}

\usepackage[capitalize,noabbrev]{cleveref}

\theoremstyle{plain}
\newtheorem{theorem}{Theorem}[section]
\newtheorem{proposition}[theorem]{Proposition}
\newtheorem{lemma}[theorem]{Lemma}
\newtheorem{corollary}[theorem]{Corollary}

\theoremstyle{definition}
\newtheorem{definition}[theorem]{Definition}
\newtheorem{assumption}[theorem]{Assumption}

\theoremstyle{remark}
\newtheorem{remark}[theorem]{Remark}

\tcolorboxenvironment{theorem}{
    enhanced,
    colback=theorbg,
    colframe=theorframe,
    boxrule=0pt,
    arc=0pt,
    left=2mm, right=2mm, top=1.5mm, bottom=1.5mm,
    borderline west={2pt}{0pt}{theorframe},
}

\tcolorboxenvironment{remark}{
    enhanced,
    colback=takeawaybg,
    colframe=takeawaybg,
    boxrule=0pt,
    arc=0pt,
    left=4mm, right=2mm, top=2mm, bottom=2mm,
    borderline west={3pt}{0pt}{takeawaystroke},
    overlay={
        \node[
            anchor=center,
            fill=white,
            draw=takeawaystroke,
            line width=1pt,
            circle,
            inner sep=2pt,
            yshift=-8pt,
            xshift=1.5pt
        ] at (frame.north west) {\color{takeawaystroke}\footnotesize\faLightbulb};
    }
}

\title{PASS: Certified Subset Repair for Classical and Quantum \\ Pairwise Constrained Clustering}

\author[1]{Pedro Chumpitaz-Flores}
\author[2]{My Duong}
\author[3]{Ying Mao}
\author[4]{Kaixun Hua}

\affil[1]{University of South Florida}
\affil[2]{Fordham University}

\begin{document}
\maketitle

\begin{abstract}
Pairwise-constrained clustering incorporates side information through must-link (ML) and cannot-link (CL) relations between samples. While these constraints can improve cluster quality, they complicate optimization at scale and limit quantum and hybrid approaches through the size of the encoded problem. PASS is a scalable framework for pairwise-constrained \(k\)-means that concentrates optimization on a small working subset while updating remaining assignments through re-centering.
Cannot-link feasibility under subset-restricted updates is formalized as a list-coloring problem on the induced constraint subgraph, yielding a checkable repair certificate with verifiable outcomes. The same subset restriction produces reduced classical subproblems and smaller quantum formulations, enabling a reduction-based hybrid evaluation under a simulation protocol. Infeasible constraint sets are handled explicitly: the pipeline returns a verifiable repair under stated conditions or reports residual conflicts under the same evaluation protocol. Across diverse benchmarks, PASS attains competitive SSE with lower runtime and returns solutions on instances where strong baselines do not finish within a fixed time budget.
\end{abstract}

\section{Introduction}

Clustering is a core task in unsupervised learning that seeks cohesive and well separated groups in data \citep{jain_data_2010}. A standard objective is the Minimum Sum of Squares Criterion (MSSC), which chooses assignments and centroids to minimize the sum of squared errors (SSE) \citep{spath_cluster_1980}. In many applications, purely unsupervised solutions may conflict with domain knowledge or operational constraints \citep{basu_constrained_2008,brieden_constrained_2017}. Pairwise supervision addresses this by imposing must-link (ML) and cannot-link (CL) relations, turning clustering into a semi supervised task that can improve interpretability and usefulness \citep{tian_model-based_2021,yang_analyzing_2022}.

Pairwise constraints also change the computational landscape. They disrupt the structure exploited by standard \(k\)-means updates, and constrained MSSC is NP hard \citep{basu_constrained_2008,brieden_constrained_2017}. Constraint sets may be inconsistent: after contracting ML components, the induced CL graph may admit no feasible \(K\)-coloring. This raises a guiding question: \emph{which points must be reconsidered to resolve cannot-link conflicts while leaving the remaining assignments fixed, and when can feasibility be certified under such restricted updates?}

\begin{figure}[t]
    \centering
    \includegraphics[width=0.95\linewidth]{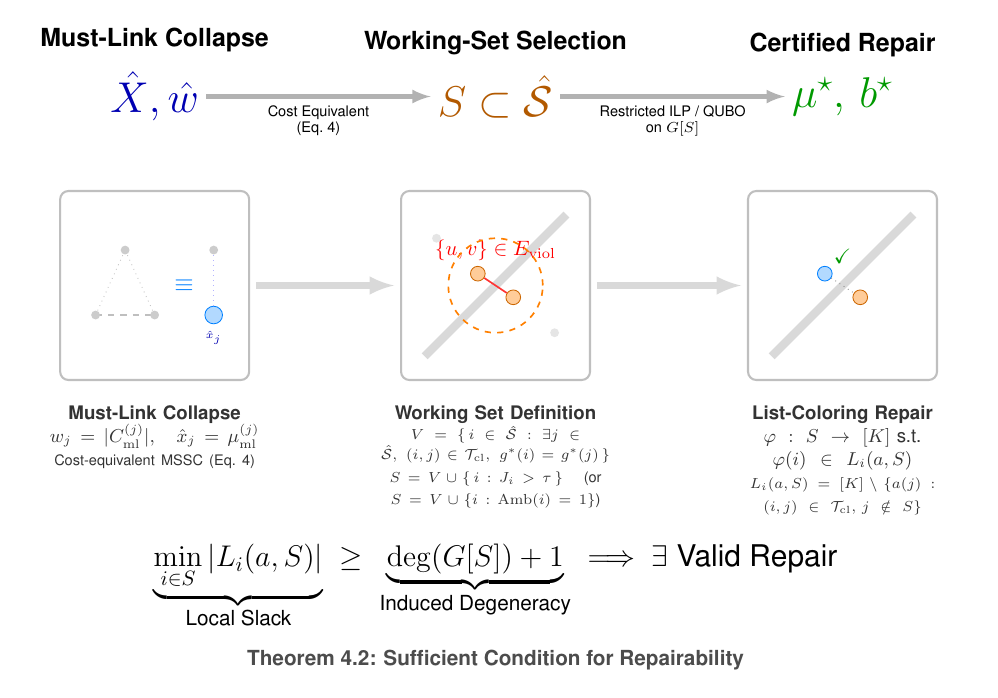}
    \vspace{-0.2cm}
    \caption{\textbf{The PASS framework: contraction, selection, and certified repair.}
    Phase 1 collapses ML components into weighted pseudo points, preserving the MSSC objective up to an additive constant (Eq.~4).
    Section~\ref{sec:subset} selects a working set \(S\) that contains all CL violations and ambiguity points.
    Section~\ref{sec:certified_repair} treats feasibility on \(S\) as a list coloring problem on the induced constraint graph \(G[S]\).
    The Local Slack Certificate (Thm.~\ref{thm:local_slack_main}) gives a sufficient condition for efficient repair based on subgraph degeneracy.}
    \label{fig:pass_framework}
\end{figure}

Existing methods address these issues with different trade offs. Heuristics can reduce violations but are sensitive to initialization and may fail to enforce feasibility on dense constraints \citep{lloyd_least_1982}. Variants refine the assignment step \citep{tan_improved_2010}, introduce constraint aware auxiliary centroids \citep{huang_semi-supervised_2008}, or enforce constraints after an initial clustering \citep{nghiem_constrained_2020}. Exact methods for MSSC and semi supervised variants prioritize solution quality but often scale poorly, and dense pairwise constraints can make instances intractable in practice \citep{aloise_np-hardness_2009,piccialli_exact_2022,chumpitaz2025scalable}.

Data reduction provides another route to improve scalability. Coresets compress the dataset into a smaller summary so that objective values on the summary approximate those on the full data \citep{schmidt2021coresets,huang2025coresets}. Related strategies concentrate computation on ambiguous points, either through soft penalties \citep{basu_active_2004} or by querying an oracle for a small number of additional labels \citep{xu2005active,xiong2016active}. Hard cannot link constraints complicate this picture because they introduce nonlocal dependencies that can increase sensitivity and inflate summary sizes in the worst case.

Reduction is also relevant for quantum and hybrid quantum classical pipelines. Shrinking the pairwise interaction graph reduces encoding and embedding overhead when dense logical couplings must be mapped onto sparse hardware topologies \citep{K_nz_2021}. This helps mitigate resource scaling limits on current devices \citep{Mirkarimi_2024,Glover2018}. Reduction strategies therefore provide a path to handle dense constraint graphs that challenge quantum annealing and related approaches \citep{Tomesh2021,Qu2022}.

\textbf{Contributions.}
PASS is a scalable framework for pairwise-constrained \(k\)-means based on subset-restricted optimization and certified repair.

\begin{itemize}[leftmargin=*, topsep=2pt, itemsep=2pt, parsep=0pt]
\item \textbf{Certified repair:} Subset-restricted updates under cannot-link constraints are framed as a repair problem. When only a working subset is allowed to change, feasibility on the induced constraint subgraph is treated as a list-coloring problem with a checkable certificate and a verifiable witness under sufficient conditions.

\item \textbf{Framework and reductions:} PASS combines must-link collapse into a weighted MSSC reduction (Eq.~\ref{eqn:overall_unc}), working-set selection, and certified subset repair for cannot-link feasibility. The same subset restriction yields reduced subproblems, including an \(\mathcal{O}(|S|k)\) QUBO formulation.

\item \textbf{Empirical evidence:} Benchmarks show competitive SSE with reduced runtime, and effectiveness when strong baselines do not finish within a fixed time budget. Certificate coverage is reported in Appendix~\ref{app:empirical_certificate_coverage}, and a reduction-enabled hybrid quantum pipeline is evaluated under a simulation protocol.
\end{itemize}

\subsection{Literature Review}
\label{sec:related_works}

This section reviews prior work on pairwise-constrained MSSC, scalable reductions, methods that focus computation on uncertain assignments, and constrained quantum clustering.

\textbf{Pairwise-constrained MSSC.}
Constrained \(k\)-means \citep{wagstaff_constrained_2001} extends Lloyd-style updates with must-link (ML) and cannot-link (CL) relations, but feasibility and initialization remain difficult in practice \citep{piccialli_exact_2022,xu_power_2019}. Heuristic variants include ICOP \(k\)-means \citep{tan_improved_2010,rutayisire_modified_2011}, MLC \(k\)-means \citep{huang_semi-supervised_2008}, and post-processing schemes that enforce constraints after an initial clustering \citep{nghiem_constrained_2020}. Exact formulations improve solution quality but scale poorly as the constraint graph becomes dense \citep{piccialli_exact_2022}.

\textbf{Scalability through reduction.}
Coresets and subsampling reduce dependence on the dataset size \(n\) while preserving constrained clustering objectives. \citet{bhattacharya2018faster} give a small candidate list of center sets, one of which achieves a \((1+\varepsilon)\) approximation. In low-dimensional Euclidean settings, \citet{schmidt2021coresets} construct constraint-agnostic coresets that preserve oracle-checkable pairwise constraints up to \((1+\varepsilon)\). \citet{braverman2021efficient} derive a uniform-sampling reduction to ring instances with summaries whose size can be independent of \(n\). Hard CL constraints, however, introduce nonlocal dependencies that can increase sensitivity and enlarge the required summary.

\textbf{Focused updates and repair.}
Penalty-based methods such as PCKMeans encode ML and CL terms in the objective, allowing controlled violations while encouraging feasibility \citep{basu_active_2004,davidson_clustering_2005}. Active-query methods reduce supervision by requesting labels for uncertain points and reclustering with the new constraints \citep{xu2005active,xiong2016active}. Repair-based methods instead act directly on violations. PCCC adjusts candidate labels for high-penalty points and propagates updates through the constraint graph, with reported scaling to \(60{,}000\) samples \citep{baumann_algorithm_2024}. Since strict and penalty-based methods optimize different tradeoffs, evaluations often report both SSE and the number of violated constraints.

\textbf{Quantum and hybrid constrained clustering.}
Pairwise constraints are difficult to encode in quantum and hybrid clustering because dense couplings increase embedding and penalty overhead on near-term hardware. Most quantum clustering work studies unconstrained \(k\)-means, and only a small set of annealer-based studies include explicit ML and CL structure \citep{seong2025hamiltonian,cohen2020ising}. Reducing optimization to a smaller constrained subproblem can make QUBO formulations more tractable and simplify penalty selection relative to full encodings.
\section{Pairwise Constrained $k$-Means} 

\subsection{Problem Formulation} 
Given a dataset \(X=\{x_1,\ldots,x_n\}\subset\mathbb{R}^D\) with $n$ samples and \(D\) features, the MSSC problem with pairwise constraints aims to find \(K\) clusters that minimize the Sum of Squared Errors (SSE) while satisfying must-link (ML) and cannot-link (CL) constraints:
\begingroup
  \setlength{\jot}{1pt}
  \begin{subequations}\label{eqn:obj}
    \begin{alignat}{3}
      &\min_{\mu,b}
        &&\sum_{i\in\mathcal S}\sum_{k\in\mathcal K}b_{i,k}\,\|x_i-\mu_k\|^2\\
      &\text{s.t. }
        &&\sum_{k\in\mathcal K} b_{i,k}=1,
          &&\forall\,i\in\mathcal S,\label{eqn:obj:a}\\
      &{}
        &&b_{i,k}=b_{i',k},
          &&\forall\,(i,i')\in\mathcal T_{ml},\,k\in\mathcal K,\label{eqn:obj:b}\\
      &{}
        &&b_{i,k}+b_{i',k}\le1,
          &&\forall\,(i,i')\in\mathcal T_{cl},\,k\in\mathcal K,\label{eqn:obj:c}\\
      &{}
        &&b_{i,k}\in\{0,1\},
          &&\forall\,i\in\mathcal S,\,k\in\mathcal K. \label{eqn:obj:d}
    \end{alignat}
  \end{subequations}
\endgroup
\noindent 
Here \(\mathcal{S}:=\{1,\cdots,n\}\) is the sample index set, \(\mathcal{K}:=\{1,\cdots,K\}\) is the cluster index set, and \(\mu=[\mu_1,\cdots,\mu_K]\) with \(\mu_k\in\mathbb{R}^D\) denotes the cluster centers. The binary variable $b_{i,k}$ is 1 when $x_i$ is assigned to cluster $k$ and 0 otherwise. The sets $\mathcal{T}_{ml}\subseteq\mathcal{S}\times\mathcal{S}$ and $\mathcal{T}_{cl}\subseteq\mathcal{S}\times\mathcal{S}$ specify pairs that must or must not be in the same cluster, respectively.

\medskip\noindent
For later use, it is convenient to introduce auxiliary variables that represent assigned squared distances. The pairwise-constrained MSSC can thus be written as:
\begin{subequations}\label{eqn:overall}
\begin{align}
\min_{\mu,\delta,b}\;&\sum_{i\in\mathcal{S}} \delta_{i,*}\label{eqn:overall:obj}\\[2pt]
\text{s.t. }\;
&-M(1-b_{i,k})\le \delta_{i,*}-\delta_{i,k} \nonumber \\
&\quad \le M(1-b_{i,k}),
\quad \forall i\in\mathcal S,\,k\in\mathcal K, \label{eqn:overall:bigM}\\[2pt]
&\delta_{i,k}\ge\|x_i-\mu_k\|^2_2,
\quad \forall i\in\mathcal S,\,k\in\mathcal K,\label{eqn:overall:dis}
\\[2pt]
& \mbox{Constraints~\ref{eqn:obj:a}--\ref{eqn:obj:d}.}
\end{align}
\end{subequations}
Here $\delta_{i,k}$ lower-bounds the squared distance between $x_i$ and $\mu_k$, $\delta_{i,*}$ is the squared distance from $x_i$ to its assigned centroid, and $M$ is a big-$M$ constant. Define $\delta_i=[\delta_{i,1},\dots,\delta_{i,K},\delta_{i,*}]$, $\delta=[\delta_1,\dots,\delta_n]$, $b_i=[b_{i,1},\dots,b_{i,K}]$, $b=[b_1,\dots,b_n]$. Constraint~\eqref{eqn:overall:bigM} links $\delta_{i,*}$ and $\delta_{i,k}$ when $b_{i,k}=1$. Moreover, \eqref{eqn:overall:dis} is tight for the assigned cluster at optimality because $\delta_{i,*}$ is minimized and linked to $\delta_{i,k}$ when $b_{i,k}=1$.

In particular, when centers $\mu$ are fixed (as in the assignment step), the costs $\|x_i-\mu_k\|_2^2$ are constants and \eqref{eqn:overall} reduces to a purely discrete assignment problem under ML/CL constraints. In this case, one may choose, e.g.,
\[
M \;\ge\; \max_{i\in\mathcal S,\,k\in\mathcal K}\|x_i-\mu_k\|_2^2,
\]
so that \eqref{eqn:overall:bigM} correctly enforces $\delta_{i,*}=\delta_{i,k}$ for the assigned $k$.

\medskip\noindent

\begin{remark}[Infeasibility and violations]
The hard constraint formulation \eqref{eqn:obj} may be infeasible due to inconsistent constraints. In experiments, we report the violation count when methods yield approximate or time-limited solutions. The certificates developed later apply when the frozen outside assignment is feasible; otherwise, we report remaining violations and/or obstruction evidence.
\end{remark}

\subsection{Cost-Equivalent MSSC with Must-Link Collapse}
\label{sec:ml-collapse}

We adopt the must-link collapse from \cite{chumpitaz2025scalable}. In our
implementation, each must-link component is represented by a single weighted
pseudo-point, rather than by repeated pseudo-samples.

Let $\mathcal{C}=\{x_1,\dots,x_p\}$ be a cluster with centroid $\mu$. Assume
$x_1,\dots,x_t$ form a must-link component $\mathcal{C}_{ml}\subseteq\mathcal{C}$.
Define
\[
\mu_{ml}=\frac{1}{t}\sum_{i=1}^t x_i,
\qquad
\mathrm{tr}(\Sigma_{ml})=\frac{1}{t-1}\sum_{i=1}^t \|x_i-\mu_{ml}\|^2,
\]
and set $\mathrm{tr}(\Sigma_{ml})=0$ when $t=1$. Let $\hat{\mathcal{C}}$ be
obtained by replacing $\mathcal{C}_{ml}$ with a pseudo-representation at
$\mu_{ml}$. Then
\begin{equation}
\label{eq:ml_sse_decomp}
\mathsf{sse}_{\mathcal{C}}(\mu)
=
\mathsf{sse}_{\hat{\mathcal{C}}}(\mu)
+
(t-1)\mathrm{tr}(\Sigma_{ml}),
\end{equation}

Let $\{\mathcal{C}^{(j)}_{ml}\}_{j\in J}$ be the must-link components in $X$ with
sizes $t_j$, centroids $\mu^{(j)}_{ml}$, and traces $\mathrm{tr}(\Sigma^{(j)}_{ml})$.
We form the collapsed dataset
\[
\hat X=(X\setminus R)\cup\bigcup_{j\in J}\{(\mu^{(j)}_{ml},w_j=t_j)\},
\qquad
R=\bigcup_{j\in J}\mathcal{C}^{(j)}_{ml}.
\]
Let $\hat{\mathcal{S}}$ index $\hat X$, and define weights $w_i$ for all $i\in\hat{\mathcal S}$ (with $w_i=1$ for original points and $w_i=t_j$ for pseudo-points). CL constraints induce relations between collapsed representatives; denote the induced set by $\hat{\mathcal T}_{cl}\subseteq \hat{\mathcal S}\times \hat{\mathcal S}$.

The resulting cost-equivalent problem is
\begin{subequations}
\label{eqn:overall_unc}
\begin{align}
\min_{\mu,\delta,b}\quad
& \sum_{i\in\hat{\mathcal{S}}} w_i\, \delta_{i,*}
+\sum_{j\in J}(t_j-1)\mathrm{tr}(\Sigma^{(j)}_{ml}),
\label{eqn:overall_unc:obj}
\\
\text{s.t.}\quad
& \text{Constraints~\ref{eqn:overall:bigM}--\ref{eqn:overall:dis} (with $i\in\hat{\mathcal S}$),}
\nonumber\\
& \text{and Constraints~\ref{eqn:obj:a}, \ref{eqn:obj:c}, \ref{eqn:obj:d} (with $\mathcal T_{cl}$ replaced by $\hat{\mathcal T}_{cl}$).}
\label{eqn:overall_unc:con}
\end{align}
\end{subequations}
where ML constraints are removed by construction (each ML component is represented by a single pseudo-point). The second term in \eqref{eqn:overall_unc:obj} is constant with respect to
$(\mu,b)$, so minimizers of \eqref{eqn:overall_unc} correspond to minimizers of
the ML-only restriction of Problem~\eqref{eqn:overall} under the
component-to-pseudo mapping.

\begin{remark}[Mixed ML and CL constraints]
We contract ML components to form $\hat{X}$. A CL constraint within an ML component indicates infeasibility, detectable during contraction. Valid CL constraints induce relations between components in $\hat{X}$ and define $\hat{\mathcal T}_{cl}$. We apply subset restriction to points with CL interactions or uncertain assignments.
\end{remark}

\section{Subset Selection Techniques}
\label{sec:subset}

We reduce the effective problem size by focusing computation on ambiguous assignments and on points involved in cannot-link conflicts. Many points keep stable labels under the current centers, such as points far from decision boundaries. Most cost concentrates near boundaries and around violated constraints. PASS therefore optimizes a working subset and preserves cannot-link feasibility across the subset boundary by construction; feasibility within the working set is handled in Section~\ref{sec:certified_repair}. PASS-CA is most effective when cannot-link violations are common, while PASS-IG is most effective when ambiguity is diffuse and violations are few.

\subsection{Constraint-Aware Subset Selection}
\label{ssec:basic_subset}

We focus computation on a subset \(S \subset \hat{\mathcal{S}}\) with \(|S| \ll |\hat{\mathcal{S}}|\) and preserve cannot-link feasibility when assignments outside \(S\) are frozen. Let \(g^*\) denote a baseline assignment, such as the unconstrained \(k\)-means assignment under the current centers.

\begin{definition}[Ambiguity and Violation Sets]
\label{def:amb_violation_sets}
For each point \(x_i\) with \(i \in \hat{\mathcal{S}}\), define the signed margin
\[
m(i) \;=\; \min_{g \neq g^*(i)} \Big( \,\|x_i - \mu_{g^*(i)}\|^2 \;-\; \|x_i - \mu_{g}\|^2 \,\Big).
\]
If \(m(i) > 0\), then some alternative cluster is closer than the current one under \(g^*\). If \(m(i) < 0\), then the current cluster is closer than any alternative. Values near zero indicate a near tie. For a threshold \(\tau \ge 0\), define
\[
\mathsf{Amb}(i) \;=\; \mathbb{I}\big[\, m(i) > -\tau \,\big],
\]
so that points with \(m(i) > 0\) and near-boundary points with \(-\tau < m(i) \le 0\) are included.

The cannot-link violation set under \(g^*\) is
\begin{align*}
    V = \Big\{& i \in \hat{\mathcal{S}} : \exists j \in \hat{\mathcal{S}} \text{ such that } \\
    & (i,j) \in \mathcal{T}_{cl} \text{ and } g^*(i) = g^*(j) \Big\}.
\end{align*}
\end{definition}

\paragraph{Threshold selection.}
We set \(\tau\) from the empirical distribution \(\{m(i)\}_{i\in\hat{\mathcal{S}}}\) to include boundary points while controlling subset size. Details are given in Appendix~\ref{app:threshold}.

\begin{figure}[htbp]
    \centering
    \includegraphics[width=0.8\linewidth]{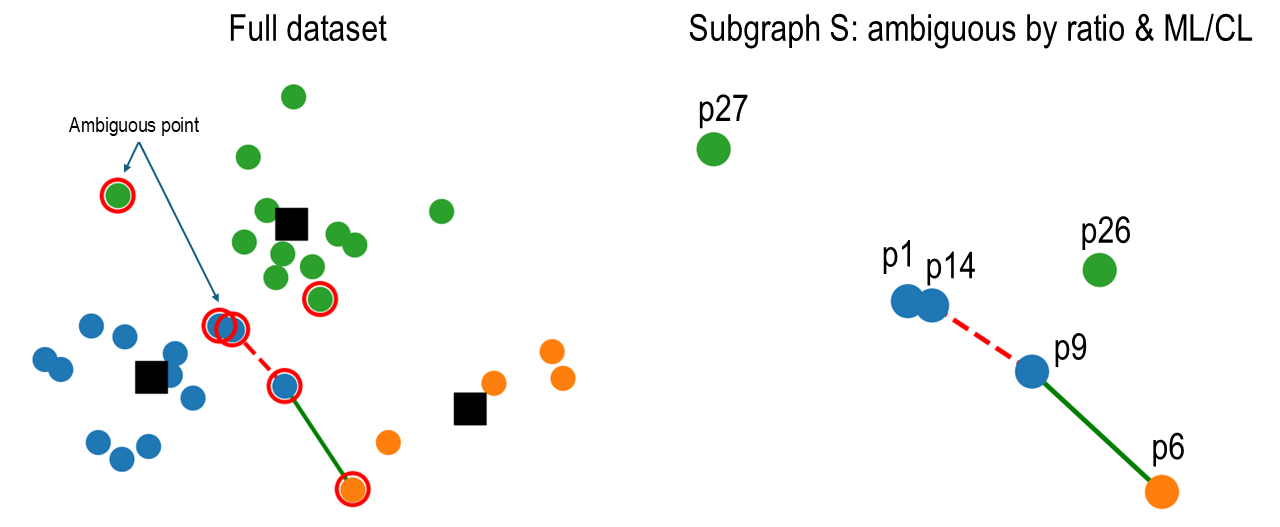}
    \caption{Subset selection with \(K = 3\). Left: ambiguous points in red and cannot-links as dashed red edges. Right: induced subproblem on \(S\) with the constraint graph.}
    \label{fig:clustering_sub}
\end{figure}

\begin{definition}[Reduced Working Subset]
\label{def:working_subset}
We select
\[
S \;=\; V \;\cup\; \left\{\, i \in \hat{\mathcal{S}} \;:\; \mathsf{Amb}(i) = 1 \,\right\}.
\]
\end{definition}

On \(S\), we freeze assignments on \(\hat{\mathcal{S}}\setminus S\) to \(g^*\) and solve the restricted assignment problem
\begin{align}
\min_{z} \quad & \sum_{i \in S} \sum_{g \in \mathcal{K}} \Delta_{i,g} \, z_{i,g}
\\
\text{s.t.} \quad
& \sum_{g \in \mathcal{K}} z_{i,g} = 1, \quad \forall i \in S \label{eq:onehot} \\
& z_{u,g} + z_{v,g} \le 1, \quad \forall (u,v) \in \mathcal{T}_{cl} \cap (S \times S), \ \forall g \in \mathcal{K} \label{eq:internal_cl} \\
& z_{u,g^*(v)} = 0, \quad \forall (u,v) \in \mathcal{T}_{cl} \text{ with } u \in S,\ v \notin S \label{eq:external_cl} \\
& z_{i,g} \in \{0,1\}, \quad \forall i \in S,\ \forall g \in \mathcal{K},
\end{align}
where \(\Delta_{i,g} = \|x_i - \mu_g\|^2 - \|x_i - \mu_{g^*(i)}\|^2\) and \(\Delta_{i,g^*(i)} = 0\).
A reference implementation is provided in Appendix~\ref{app:algorithms}. We restrict the label domain to candidate sets \(K_i \subseteq \mathcal{K}\) to reduce the number of binaries (Appendix~\ref{sec:0_1_ILP}); the boundary constraints \eqref{eq:external_cl} are enforced with respect to the frozen labels.

\begin{proposition}[Feasibility preservation across the subset boundary]
\label{prop:boundary_feas_preserve}
Any solution feasible for \eqref{eq:onehot}--\eqref{eq:external_cl}, combined with the frozen assignment \(g^*\) on \(\hat{\mathcal{S}}\setminus S\), yields a globally cannot-link feasible assignment.
\end{proposition}

\begin{proof}
By construction, every violated cannot-link edge under \(g^*\) has both endpoints in \(V\subseteq S\). Hence no cannot-link edge is violated entirely within \(\hat{\mathcal{S}}\setminus S\), so the frozen outside assignment is cannot-link feasible.

Constraints inside \(S\) are enforced by \eqref{eq:internal_cl}. For edges \((u,v)\in\mathcal{T}_{cl}\) with \(u\in S\) and \(v\notin S\), constraint \eqref{eq:external_cl} forbids assigning \(u\) to the frozen label \(g^*(v)\), so cross-boundary edges are satisfied. Therefore the combined assignment is globally cannot-link feasible.
\end{proof}

\subsection{Information Geometric Subset Selection}
\label{ssec:igass}

The selection above preserves boundary feasibility: when the restricted problem is feasible, its solution extends to a globally cannot-link feasible assignment with frozen outside labels. We also consider a budgeted selection rule based on uncertainty scores from soft assignments.

\begin{definition}[Fisher--Rao Ambiguity Score]
\label{def:fr_score}
Let \(p_i \in \Delta^{K-1}\) be a soft assignment vector, obtained from current centers via a softmax over squared distances with temperature \(T>0\),
\(p_i^{(g)} \propto \exp(-\|x_i-\mu_g\|^2/T)\).
Let \(g_1\) and \(g_2\) be the indices with the largest probabilities in \(p_i\), with deterministic tie breaking, and define
\[
q_i \;=\; \left( \frac{p_i^{(g_1)}}{p_i^{(g_1)} + p_i^{(g_2)}} \,,\;
\frac{p_i^{(g_2)}}{p_i^{(g_1)} + p_i^{(g_2)}} \right), \qquad
b \;=\; \left( \tfrac{1}{2} , \tfrac{1}{2} \right).
\]
Define
\begin{align*}
d_{\mathrm{FR}}(q_i , b) &= 2 \arccos \!\left( \frac{1}{\sqrt{2}} \left( \sqrt{q_{i1}} + \sqrt{q_{i2}} \right) \right), \\
J_i &= 1 - \frac{2}{\pi} \, d_{\mathrm{FR}}(q_i , b) \in [0,1].
\end{align*}
\end{definition}

\begin{definition}[Budgeted Information Geometric Subset]
\label{def:budgeted_ig_subset}
Given a budget \(m\) with \(m \ge |V|\), we select
\(
S^* \;=\; V \;\cup\; \{\, i_1 , \ldots , i_{\, m - |V|} \,\},
\)
where the indices \(i_1 , i_2 , \ldots\) belong to \(\mathcal{S} \setminus V\) and have the largest values of \(J_i\).
\end{definition}

On \(S^*\) we solve \eqref{eq:onehot}--\eqref{eq:external_cl}; only the subset rule changes, replacing margins with Fisher--Rao scores. Boundary feasibility preservation follows as in Proposition~\ref{prop:boundary_feas_preserve} because \(V \subseteq S^*\) and the boundary constraints are enforced. Details are given in Appendix~\ref{app:ig_details}.

Figure~\ref{fig:fisher_rao_subset} illustrates this selection process.

\begin{figure}[htbp]
    \centering
    \includegraphics[width=0.9\linewidth]{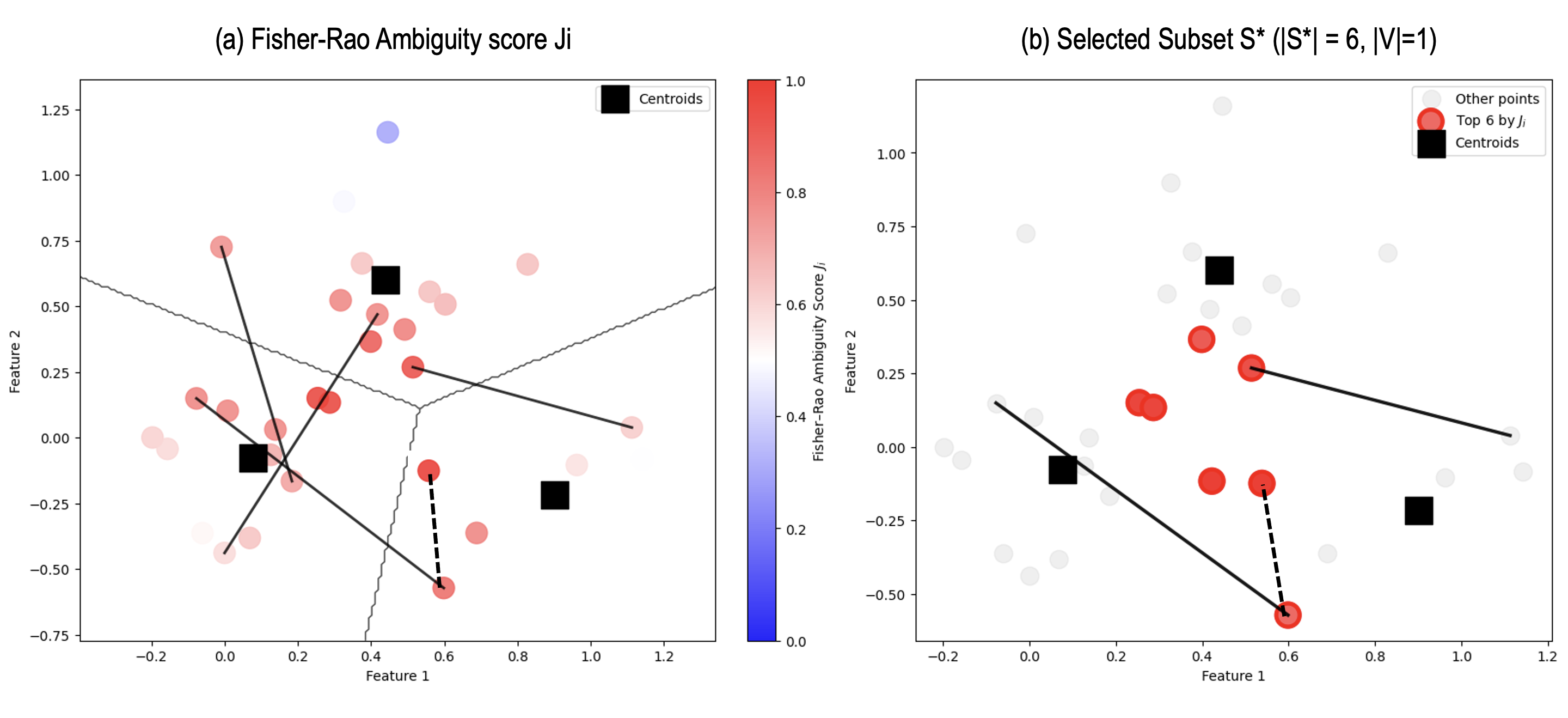}
    \caption{Ambiguity-based subset selection using Fisher--Rao geometry. \textbf{(a)} Fisher--Rao ambiguity scores $J_i$ for each data point (red = ambiguous, blue = certain). Centroids (black squares) and decision boundaries (gray contours) illustrate the clustering structure. Cannot-link constraints are represented by solid black lines, with a constraint violation shown as a dashed black line. \textbf{(b)} Selected subset $S^*$ ($|S^*|=6$, $|V|=1$) including the constraint violation and the most ambiguous points.}
    \label{fig:fisher_rao_subset}
\end{figure}

\section{Certified Reveal and Repair for Cannot-Link Constraints}
\label{sec:certified_repair}

PASS restricts optimization to a working (reveal) set \(S\), inducing a repair subproblem: assignments outside \(S\) are frozen and only points in \(S\) may change labels. The results in this section apply when the frozen outside assignment is cannot-link feasible. If a violated cannot-link edge lies entirely in \(\mathcal{S}\setminus S\), then no update restricted to \(S\) can remove that violation; in this case PASS reports the remaining violations and, when available, obstruction evidence.

Let \(G=(\mathcal{S},\mathcal{T}_{cl})\) be the cannot-link graph and let \(a:\mathcal{S}\to\mathcal{K}\) be a current assignment. Given a reveal set \(S\subseteq\mathcal{S}\), points in \(\mathcal{S}\setminus S\) are frozen to \(a\). For each \(i\in S\), define the labels allowed by frozen neighbors,
\[
L_i(a,S) \;=\; \mathcal{K}\setminus \{\, a(j) \;:\; (i,j)\in\mathcal{T}_{cl},\ j\notin S \,\}.
\]
For each edge \((i,j)\in\mathcal{T}_{cl}\) with \(i\in S\) and \(j\notin S\), forbidding the frozen label \(a(j)\) matches the boundary constraint \eqref{eq:external_cl} in Section~\ref{sec:subset}.

\paragraph{List coloring view.}
Repair corresponds to recoloring \(G[S]\) subject to frozen labels outside \(S\). The following lemma states the equivalence.

\begin{lemma}[Repairability as list coloring]
\label{lem:repair_list_main}
Assume the frozen assignment is feasible, meaning
\(a(u)\neq a(v)\) for all \((u,v)\in\mathcal{T}_{cl}\) with \(u,v\notin S\).
Then a globally feasible repaired assignment \(a'\) that differs from \(a\) only on \(S\) exists
if and only if the induced subgraph \(G[S]\) admits a proper list coloring
\(\phi:S\to\mathcal{K}\) with \(\phi(i)\in L_i(a,S)\) for all \(i\in S\).
\end{lemma}

\paragraph{A checkable repair certificate.}
A sufficient condition can be verified a posteriori on the reveal set and yields a repair and witness trace.

\begin{theorem}[Local slack certificate for repair]
\label{thm:local_slack_main}
Assume the frozen assignment outside \(S\) is feasible.
If
\[
\min_{i\in S} |L_i(a,S)| \;\ge\; \mathrm{deg}(G[S]) + 1,
\]
where \(\mathrm{deg}(G[S])\) denotes the degeneracy of \(G[S]\),
then a globally feasible repair exists and can be constructed by greedy list coloring along a degeneracy ordering.
Moreover, the certificate consists of the induced lists, a degeneracy peeling trace, and the greedy coloring trace, and can be checked independently.
\end{theorem}
Empirical coverage of this certificate, including how often it applies when repair is invoked and the per-call overhead, is reported in Appendix~\ref{app:empirical_certificate_coverage} (Table~\ref{tab:certificate_coverage_app}).
Appendix~\ref{app:certified_repair} provides certified lower bounds on minimum reveal cost, a posteriori exactness certificates when a lower bound becomes tight, and certified obstruction cores when repair fails under a stated budget.
\begin{table*}[htbp]
\caption{\textbf{Reduction (subset) closes the gap to SOTA}: PASS matches PCCC’s quality on classical pairwise-constrained clustering and scales to \(\mathbf{4\mathrm{M}+}\) points (CL, ML, Both; \(k=3\))}
\label{tab:classical_results}
\resizebox{\textwidth}{!}{
\centering
\small
\begin{tabular}{llccc|llccc}
\toprule
Dataset & Method & CL & ML & Both & Dataset & Method & CL & ML & Both \\
\midrule
Iris & COP-$k$-means &  102.03 
&  \underline{93.71}&  \textbf{86.72}& 
HTRU2    & COP-$k$-means & \multicolumn{3}{c}{No solution found (100 iterations)}\\
$n=150$ & BLPKM-CC & \textbf{84.24}& \textbf{83.72}& 86.83&
n = 17,898
& BLPKM-CC & 1.03E+08
& 1.36E+08& 1.44E+08\\
$d=4$ & PCCC&  \underline{84.35}&  \textbf{83.72} 
&  \underline{86.75}&
d = 8& PCCC& \textbf{9.21E+07}
& 1.23E+08& 1.33E+08\\
 & PASS-CA&  \textbf{84.24} 
&  \textbf{83.72} 
&  \underline{83.75}& 
& PASS-CA& 9.35E+07 (22)& \textbf{1.22E+08}& 1.33E+08 (18)\\
& PASS-IG&  \textbf{84.24} 
&  \textbf{83.72} 
&  \textbf{86.72}&
& PASS-IG& \underline{9.22E+07}& \textbf{1.22E+08}& \textbf{1.32E+08}\\
\midrule
Seeds & COP-$k$-means &  655.55 
&  658.05 
&  656.82 
&
AC
& COP-$k$-means & \multicolumn{3}{c}{No solution found (100 iterations)}\\
$n=210$ & BLPKM-CC & \underline{600.36}& \underline{629.29}& \underline{633.91}&
$n=53,413$
& BLPKM-CC & 2.54E+03
& \underline{2.38E+03}& 2.60E+03
\\
$d=7$ & PCCC&  \underline{600.36}&  \textbf{627.66} 
&  \textbf{633.25} 
&
$d=23$& PCCC& \textbf{2.24E+03}
& \textbf{2.02E+03}
& \textbf{2.38E+03}
\\
 & PASS-CA&  \underline{600.36}&  \textbf{627.66} 
&  634.88 
& & PASS-CA& 2.25E+03 (9)& \textbf{2.02E+03}
&2.51E+03
(20)\\
& PASS-IG&  \textbf{600.18} 
&  \textbf{627.66} 
&  \underline{633.91}&
& PASS-IG& \underline{2.28E+03}& \textbf{2.02E+03}& \underline{2.52E+03}\\
\midrule
Hemi& COP-$k$-means & 1.75E+07
& 2.44E+07
& N/A &
Skin& COP-$k$-means & \multicolumn{3}{c}{No solution found (100 iterations)}\\
n = 1,955& BLPKM-CC & 1.56E+07
& \underline{2.05E+07}& 2.36E+07
&
$n = 245,057$& BLPKM-CC & \underline{9.33E+08}& 1.60E+09& 1.63E+09\\
d = 7& PCCC& \textbf{1.36E+07}
& \textbf{1.81E+07}
& \textbf{2.10E+07}
&
d = 3& PCCC& \textbf{9.30E+08}& 1.61E+09& \textbf{1.60E+09}
\\
 & PASS-CA& 1.86E+07 & \textbf{1.81E+07} & 2.17E+07 (11)& 
 & PASS-CA& 1.02E+09& \textbf{1.56E+09}& 1.73E+09
\\
& PASS-IG& \underline{1.38E+07}& \textbf{1.81E+07}
& \underline{2.12E+07}&
& PASS-IG& \textbf{9.30E+08}& \textbf{1.56E+09}& \underline{1.61E+09}\\
\midrule
PR2392& COP-$k$-means & N/A & 3.68E+10
& N/A &
Gas\_methane& COP-$k$-means & \multicolumn{3}{c}{No solution found (100 iterations)}\\
n = 2,392& BLPKM-CC & \underline{2.60E+10}& \underline{3.40E+10}& \underline{3.77E+10}&
n = 4,178,504& BLPKM-CC & \multicolumn{3}{c}{No solution found (1h time limit)}\\
d = 2& PCCC& 2.77E+10
& \textbf{3.34E+10}
& \textbf{3.74E+10}
&
d = 18& PCCC& \multicolumn{3}{c}{No solution found (1h time limit)}\\
& PASS-CA& 2.59E+10 (3)& \textbf{3.34E+10} & 3.82E+10 (16)& 
& PASS-CA& \multicolumn{3}{c}{No solution found (1h time limit)}\\
& PASS-IG& \textbf{2.59E+10}
& \textbf{3.34E+10}
& 3.81E+10&
& PASS-IG& \textbf{3.13E+14}& \textbf{2.67E+14}& \textbf{4.56E+14}\\
\midrule
RDS\_CNT& COP-$k$-means & N/A & 8.51E+07
& N/A &
Gas\_CO& COP-$k$-means & \multicolumn{3}{c}{No solution found (100 iterations)}\\
n = 10,000& BLPKM-CC & 3.00E+07
& 6.83E+07
& 8.22E+07
&
n = 4,208,261& BLPKM-CC & \multicolumn{3}{c}{No solution found (1h time limit)}\\
d = 3& PCCC& \textbf{2.93E+07}
& 6.23E+07& \textbf{7.67E+07}
&
d = 18& PCCC& \multicolumn{3}{c}{No solution found (1h time limit)}\\
 & PASS-CA& 2.96E+07 (9)& \textbf{6.22E+07} & 8.18E+07 (66)&
 & PASS-CA& \multicolumn{3}{c}{No solution found (1h time limit)}\\
& PASS-IG& \underline{2.94E+07}& \textbf{6.22E+07}& \underline{7.93E+07}&
& PASS-IG& \textbf{3.71E+14}& \textbf{3.21E+14}& \textbf{1.78E+14}\\
\bottomrule
\end{tabular}
}
\begin{flushleft}\footnotesize
Notes: Parentheses indicate the number of violated constraints. Best per baseline is highlighted in \textbf{bold}, second best per baseline is \underline{underlined}. 1h time limit
\end{flushleft}
\end{table*}

\section{Pairwise Constrained Subset Selection Framework}
\label{sec:pairwise_aware}

PASS (\textbf{P}airwise constrained \textbf{A}ssignment via \textbf{S}ubset \textbf{S}election) is a pairwise-constrained $k$-means framework that concentrates computation on ambiguity regions and cannot-link conflicts (Algorithm~\ref{alg:framework}). Subset selection is described in Section~\ref{sec:subset}, and Section~\ref{sec:certified_repair} gives certified reveal and repair guarantees. PASS uses two selectors: \textbf{PASS-CA}, which targets violated cannot-link edges and ambiguity points by margin, and \textbf{PASS-IG}, which uses Fisher--Rao uncertainty. The working-set size is \(m=\min(\alpha n,\;|V|+\beta K\log n)\); working-set statistics under the same budget are reported in Appendix~\ref{app:ig_vs_ca_workingset}.
At each iteration, labels are updated on the working set \(S\) by solving a restricted 0--1 ILP over candidate label sets \(K_i\). Decision variables \(z_{i,g}\) indicate whether point \(i\in S\) takes label \(g\in K_i\). The objective minimizes the sum of squared distances to current centroids over \(S\). Constraints assign one label per \(i\in S\), enforce cannot-link constraints within \(S\), and forbid labels that would violate cannot-link edges to frozen neighbors in \(\hat{\mathcal{S}}\setminus S\). The solver is warm started from the current labels, with a deterministic local search fallback when a mixed integer solver is unavailable. The ILP has \(O(mG)\) binary variables; details are in Appendix~\ref{sec:0_1_ILP}.

\subsection{Computational Experiments}
\label{sec:experiments}

Experiments were run on Ubuntu Linux using Gurobi 11.0.2 and Python 3.10.12. Twelve datasets with sample size \(n\) and dimensionality \(D\) were taken from the UCI Repository~\citep{Dua:2019}. Pairwise constraints follow random pair sampling as in prior work~\citep{piccialli_exact_2022, wagstaff_constrained_2001, guns_repetitive_2016}. For each dataset, three settings are evaluated: ML, CL, and Both. In ML and CL, \(\frac{n}{4}\) constraints are added; in Both, \(\frac{n}{4}\) ML and \(\frac{n}{4}\) CL constraints are added. Datasets are grouped as \emph{small} (\(n\le 1{,}000\)), \emph{medium} (\(n\le 10{,}000\)), \emph{large} (\(n\le 100{,}000\)), and \emph{huge} (\(n\ge 100{,}000\)).

The sum of squared errors (SSE) and the number of violated constraints are reported under ML, CL, and Both, and results are interpreted feasibility first. When instances are infeasible or optimization is time-limited, violation counts are reported alongside SSE. ``No solution found'' indicates that a method did not return a solution within the budget, so SSE is not comparable then. Runtime and analysis are provided in Appendix~\ref{sec:comp_analysis}, including a working-set comparison between PASS-IG and PASS-CA (Appendix~\ref{app:ig_vs_ca_workingset}).
PASS is compared to \textsc{COP} $k$-means~\citep{wagstaff_constrained_2001} (100 restarts), \textsc{BLPKM} CC~\citep{baumann2020binary} (assignment per iteration), and \textsc{PCCC}~\citep{baumann_algorithm_2024}. Best values are highlighted in Table~\ref{tab:classical_results}. Comparisons prioritize feasibility (violations), then SSE among feasible solutions, then runtime. On small and medium datasets, PASS matches baseline SSE, and on large datasets PASS returns solutions within the stated budgets when other methods often do not return results.

\begin{algorithm}[!t]
\caption{PASS $k$-means framework}
\label{alg:framework}
\begin{algorithmic}[1]
\Require data $X$, number of clusters $K$, ML and CL constraints
\State \textbf{Phase 1: Must link contraction}
\State \hspace{1em} Collapse must link components into weighted pseudo points.
\State \textbf{Phase 2: Initialization}
\State \hspace{1em} Run minibatch $k$-means on the contracted data to obtain initial centroids.
\State \textbf{Phase 3: Subset refinement}
\Repeat
  \State Select a working set $S$ using PASS-CA or PASS-IG (Section~\ref{sec:subset}).
  \State Reassign points in $S$ by solving a restricted 0--1 ILP over candidate sets $K_i$, enforcing cannot-link constraints on $S$ and compatibility with frozen neighbors outside $S$.
  \State Update centroids from the new assignments.
\Until{convergence or the iteration budget is reached}
\State \textbf{Phase 4: Certified reveal and repair}
\State \hspace{1em} Repair cannot-link violations and output a checkable witness when available (Section~\ref{sec:certified_repair}).
\State \hspace{1em} Lift assignments from pseudo points back to the original samples.
\State \Return $(b,\mu)$
\end{algorithmic}
\end{algorithm}

\begin{table*}[htbp]
\caption{\textbf{SSE gains via reduction: solving a reduced (subset) problem, q-PCKMeans attains lower SSE than full-problem CP\!-QAOA under CL/ML/Both constraints under the stated simulation protocol} (\(p{=}1\), shots \(=\) 2048).}
\label{tab:sse_compact}
\resizebox{\textwidth}{!}{
\centering
\small
\begin{tabular}{llccc|llccc}
\toprule
\multicolumn{5}{c|}{Real-world datasets} & \multicolumn{5}{c}{Synthetic datasets ($d=2$)} \\
\midrule
Dataset & Method & CL & ML & Both & Dataset & Method & CL & ML & Both \\
\midrule
Iris & CP-QAOA &  642.78 (26)&   678.44 (24)& 673.57 (43)& Circles & CP-QAOA &  186.54 (47)&  187.93 (36)&  187.76 (72)\\
$n, d=(150, 4)$& q-PCKMeans (CA) & \textbf{101.16}& \textbf{106.09}& \textbf{97.62}& $n=300$ & q-PCKMeans (CA) & \textbf{140.59}& \textbf{139.58}& \textbf{165.56}\\
$k=3$& q-PCKMeans (IG) & \textbf{101.16}& \textbf{106.09}& \textbf{97.62}& $k=2$ & q-PCKMeans (IG) & 143.78& 141.62& \textbf{165.56}\\
\midrule
Seeds & CP-QAOA &  2.67E+03 (30)&  2.71E+03 (54)&  2.70E+03 (54)& Moons & CP-QAOA &  300.12 (40)&  299.30 (24)&  296.56 (81)\\
$n, d=(210,7)$& q-PCKMeans (CA) & \textbf{682.73}& \textbf{691.74}& \textbf{724.19}& $n=300$ & q-PCKMeans (CA) & \textbf{196.22}& \textbf{147.70}& \textbf{231.00}\\
$k=3$& q-PCKMeans (IG) & \textbf{682.73}& 733.93& \textbf{724.19}& $k=2$ & q-PCKMeans (IG) & 199.33& 151.30& \textbf{231.00}\\
\midrule
Haberman & CP-QAOA &  5.45E+04 (82)&  5.45E+04 (33)&  5.44E+04 (82)& Spiral & CP-QAOA &  4.10E+03 (36)&  4.09E+03 (34)&  4.08E+03 (75)\\
$n,d=(306,3)$& q-PCKMeans (CA) & \textbf{5.18E+4}& \textbf{5.25E+4}& \textbf{5.32E+4}& $n=300$ & q-PCKMeans (CA) & \textbf{2.98E+3}& \textbf{2.96E+3}& \textbf{3.60E+3}\\
$k=2$& q-PCKMeans (IG) & 5.20E+4& 5.26E+4& \textbf{5.32E+4}& $k=2$ & q-PCKMeans (IG) & 3.04E+3& 3.15E+3& \textbf{3.60E+3}\\
\midrule
Land mine & CP-QAOA &  83.90 (39)&  83.30 (50)&  83.72 (77)& Moons\_2 & CP-QAOA &  798.48 (57)&  792.71 (35)&  796.21 (81)\\
$n,d=(338,3)$& q-PCKMeans (CA) & \textbf{27.80}& \textbf{45.66}& \textbf{44.71}& $n=400$ & q-PCKMeans (CA) & \textbf{517.33}& \textbf{519.54}& \textbf{660.92}\\
$k=5$& q-PCKMeans (IG) & \textbf{27.80}& \textbf{45.66}& \textbf{44.71}& $k=2$ & q-PCKMeans (IG) & 540.14& 531.39& \textbf{660.92}\\
\midrule
Monk\_2 & CP-QAOA &  1.62E+03 (74)&  1.62E+03  (50)&  1.62E+03 (117)& An blobs & CP-QAOA &  994.76 (54)&  999.75 (53)&  997.96 (115)\\
$n,d=(432,6)$& q-PCKMeans (CA) & \textbf{1.54E+3}& \textbf{1.55E+3}& \textbf{1.58E+3}& $n=500$ & q-PCKMeans (CA) & \textbf{153.44}& \textbf{153.44}& \textbf{153.44}\\
$k=2$& q-PCKMeans (IG) & 1.55E+3& \textbf{1.55E+3}& \textbf{1.58E+3}& $k=3$ & q-PCKMeans (IG) & 154.90& 153.94& 153.90\\
\midrule
Raisin & CP-QAOA & 2.86E+12 (126)& 2.86E+12 (80)& 2.86E+12 (66)& Vd blobs & CP-QAOA &  1.20E+03 (70)&  1.20E+03 (55)&  1.19E+03  (160)\\
$n,d=(900,7)$& q-PCKMeans (CA) & \textbf{1.35E+12}& \textbf{1.37E+12}& 1.71E+12& $n=600$ & q-PCKMeans (CA) & \textbf{458.60}& \textbf{691.46}& \textbf{811.74}\\
$k=2$& q-PCKMeans (IG) & \textbf{1.35E+12}& 1.40E+12& \textbf{1.61E+12}& $k=3$ & q-PCKMeans (IG) & \textbf{458.60}& 705.03& \textbf{811.74}\\
\bottomrule
\end{tabular}
}
\begin{flushleft}\footnotesize
Notes: Parentheses indicate the number of violated constraints. Best per setting is highlighted in \textbf{bold}, 1h time limit.
\end{flushleft}
\end{table*}

\section{Extending to Quantum Clustering}
\label{sec:quantum}

Our classical subset selection framework enables scalability, but the optimization over the ambiguous subset $S$ can remain challenging due to cannot-link constraints. This makes the extracted instance a natural target for quantum heuristics. The quantum experiments in this section evaluate a reduction-enabled hybrid pipeline under a simulation protocol and do not claim quantum advantage over classical solvers.
Pairwise constraints map directly to quadratic couplings in QUBO models, which are the native input for many quantum optimizers \citep{cohen2020ising,seong2025hamiltonian}. Existing quantum approaches face a scalability barrier because they require $O(N \times k)$ variables, which is prohibitive for NISQ devices \citep{preskill2018quantum,bharti2022noisy}. Our key insight is that \textbf{our subset selection bypasses this barrier}. By identifying the critical subset $S$, we reduce the quantum variable count to $O(|S| \times k)$, which makes a resource-aware quantum formulation viable within hybrid quantum-classical workflows \citep{harrow_small_2020,tomesh_coreset_2020}.

\subsection{A Quantum Pairwise Constrained Clustering Algorithm}

Most NISQ clustering studies focus on unconstrained $k$-means, leaving explicit pairwise constraints underexplored in scalable settings \citep{cohen2020ising}. We present a quantum refinement algorithm whose complexity depends on $|S|$ rather than $N$. \textbf{This makes constrained quantum refinement practical when a full $(N\times k)$ encoding is infeasible for near-term hardware}.

Problem~\eqref{eqn:qubo} is converted to a QUBO by embedding hard constraints as penalties. Let
\(
\lambda \;=\; \sum_{i\in S}\sum_{g\in\mathcal K}\bigl|\Delta_{i,g}\bigr| \;+\;\varepsilon,\quad \varepsilon>0,
\)
which strictly upper-bounds the variation of the linear term over all assignments. Any infeasible configuration incurs at least \(\lambda\) additional energy (Proposition~\ref{prop:lambda}). We denote by \(d\) the binary assignment variables used for quantum optimization. These correspond to the subset assignment variables in Section~\ref{sec:subset}. The penalties are
\begin{subequations}\label{eqn:penalties}
\begin{align}
H_{\text{one-hot}}
&= \lambda \sum_{i\in S}\!\Bigl(1 - \sum_{g\in\mathcal{K}} d_{i,g}\Bigr)^2,\\
H_{\text{ML}}
&= \lambda \sum_{(u,v)\in\mathcal{T}_{ml}\cap(S\times S)}\sum_{g\in\mathcal{K}} (d_{u,g}-d_{v,g})^2,\\
H_{\text{CL}}
&= \lambda \sum_{(u,v)\in\mathcal{T}_{cl}\cap(S\times S)}\sum_{g\in\mathcal{K}} d_{u,g}\,d_{v,g}.
\end{align}
\end{subequations}
The full Hamiltonian is
\begin{equation}
\label{eqn:hamiltonian}
\begin{aligned}
H(d)
&= \sum_{i\in S}\sum_{g\in\mathcal{K}} \Delta_{i,g}\, d_{i,g}
   + H_{\text{one-hot}} + H_{\text{ML}} + H_{\text{CL}} + H_{\text{cross}} \\
&= d^{\top} Q\, d + c^{\top} d \; (+\text{const}).
\end{aligned}
\end{equation}
where \(H_{\text{cross}}\) groups the linear cross-edge terms for pairs with exactly one endpoint in \(S\) (details in Appendix~\ref{sec:qvf}). A classical re-centering step is followed by quantum refinement on the selected subset $S$. When the restricted instance admits a feasible assignment, the refinement step preserves feasibility for the constraints encoded in the refinement step. Otherwise, residual violations are reported under the evaluation protocol described earlier (see Remark~2.1). This formulation constitutes the quantum pairwise constrained clustering algorithm (q-PCKMeans; Appendix~\ref{sec:qvf}).

\subsection{Computational Experiments} \label{sec:quantum_experiments}

The evaluation compares q\!-PCKMeans with CP\!-QAOA~\citep{hadfield2019qaoa}, a full encoding of constrained clustering that requires $\mathcal{O}(n k)$ binary variables. This comparison isolates the impact of subset selection: full encodings become impractical on near-term devices due to limited qubits and coherence times.
Implementation uses Qiskit~1.3.2 and Python~3.10.12, following the protocol in Section~\ref{sec:pairwise_aware}. Each circuit uses a single QAOA layer ($p{=}1$) with 2{,}048 shots, following the few-shot guide of \citet{hao2024end}. All runs use Qiskit’s \texttt{AerSimulator} with IBM-Q noise emulation. We evaluate q\!-PCKMeans on 12 datasets, including 6 UCI datasets~\citep{Dua:2019} and 6 synthetic datasets. To keep costs tractable while moving beyond prior constrained quantum clustering evaluations, datasets are capped at $n \le 900$ and $d \le 7$.

Following Section~\ref{sec:pairwise_aware}, quality is measured by SSE. When instances are infeasible or optimization is time-limited, the number of violated constraints is reported, consistent with the evaluation protocol. Across benchmarks and constraint regimes, q-PCKMeans improves SSE relative to CP-QAOA (reductions of 6--84\%) while remaining competitive with classical heuristics. These gains arise from reducing the quantum variable count from $\mathcal{O}(nk)$ to $\mathcal{O}(|S|k)$ and from a pairwise-restricted subspace mixer that biases the search toward feasible assignments within the mixer constraints. This reduces reliance on penalties for infeasibility. Under the noise emulation protocol in Appendix~\ref{sec:eval}, this supports execution up to $n=900$ with low circuit depth and low violation counts in the reduced space.

\section{Discussion}
\label{sec:discussion}

\noindent\textbf{The ambiguous kernel enables scalable refinement with checkable repair outcomes.}
PASS concentrates computation on points likely to change, including endpoints of cannot-link violations and points near decision boundaries. Updates are restricted to a working set $S$, which is also used as the reveal set for repair, reducing the decision space from $O(Nk)$ to $O(|S|k)$ while maintaining centroid updates through re-centering. Cannot-link feasibility is handled through repair on the induced constraint subgraph. With labels outside $S$ frozen, repair reduces to list coloring with lists induced by frozen neighbors. Successful repair returns a repaired assignment with a witness trace that can be checked independently. If the frozen outside assignment is infeasible, or if repair fails under the selected reveal set, remaining violations or obstruction evidence are reported under the same evaluation protocol. The same reduction yields compact QUBO instances that support hybrid quantum refinement on the extracted combinatorial core.

\noindent\textbf{Quantum feasibility under reduction.}
The quantum implementation prioritizes feasibility in the reduced space over SSE improvements, reflecting NISQ constraints. This supports shallow circuits and limited shots under the simulation protocol, while the gap to classical methods reflects hardware limits and restricted circuit depth. Subset restriction reduces the number of binary variables from $O(Nk)$ to $O(|S|k)$; when $|S|$ remains compact, the resulting QUBOs stay small as $N$ grows. This suggests that improvements in qubit count, fidelity, and feasible-subspace circuit design may expand the range of constrained instances where hybrid refinement is practical, without claiming quantum advantage over classical solvers.

\section{Conclusion}
This paper presented PASS for pairwise constrained \emph{k}-means. PASS contracts must-link components into weighted pseudo-points, selects a working set $S$ from cannot-link conflicts and ambiguous points, and optimizes assignments only on $S$ with labels outside $S$ fixed. A repair step casts cannot-link enforcement as list coloring and returns a certificate when it succeeds; otherwise, residual violations or obstruction evidence are reported. Experiments show reduced runtime with comparable SSE, and the same reduction yields compact QUBOs that support q-PCKMeans, improving SSE relative to a full-encoding CP-QAOA baseline under the stated simulation protocol.




\bibliography{custom}

\clearpage

\onecolumn
\title{PASS: Certified Subset Repair for Classical and Quantum \\ Pairwise Constrained Clustering\\(Supplementary Material)}
\maketitle
\appendix
\section{Additional Details for Subset Selection}
\label{app:threshold}

This appendix provides implementation details and supporting material for
Section~\ref{sec:subset}.

\subsection{Margin threshold rule}
\label{app:threshold:margin}

Recall the signed margin
\[
m(i) \;=\; \min_{g \neq g^*(i)}\Big(\|x_i-\mu_{g^*(i)}\|^2 - \|x_i-\mu_g\|^2\Big).
\]
A threshold \(\tau \ge 0\) is selected using a percentile rule. Let
\(\theta_p\) be the \(p\)-th percentile of \(\{m(i)\}_{i\in\hat{\mathcal{S}}}\)
with \(p \in [10,30]\). Set
\[
\tau \;=\; \max\{0,\,-\theta_p\},
\qquad
A \;=\; \{\, i \in \hat{\mathcal{S}} : m(i) > -\tau \,\}.
\]
This keeps misassigned points (\(m(i) > 0\)) and points near a tie
(\(-\tau < m(i) \le 0\)). The rule controls \(|A|\) without requiring an
instance dependent absolute scale.

\subsection{Working subset size}

The working subset is \(S = V \cup A\), where
\[
V \;=\; \left\{\, i \in \hat{\mathcal{S}} \;:\; \exists j \in \hat{\mathcal{S}} \ \text{such that}\ \{i,j\}\in\mathcal{T}_{cl}
\ \text{and}\ g^*(i)=g^*(j)\,\right\}.
\]
The condition \(V \subseteq S\) is required for feasibility preservation in
Section~\ref{ssec:basic_subset}.

\subsection{Selection Algorithms}
\label{app:algorithms}

This section lists reference procedures for constructing the working set used in Section~\ref{sec:subset}. Algorithm~\ref{alg:constraint_subset_app} implements the margin-based rule \(S=V\cup A\), where \(V\) collects endpoints of violated cannot-link constraints under the baseline assignment \(g^*\) and \(A\) collects near-boundary points via the percentile threshold \(\tau\). Algorithm~\ref{alg:igass_app} implements the budgeted information-geometric rule \(S=V\cup\mathrm{Top}_{m-|V|}(J_i)\), where \(J_i\) is the Fisher--Rao ambiguity score and \(m \ge |V|\). Unless stated otherwise, \(\mathcal{T}_{cl}\) is treated as an undirected edge set.

\begin{algorithm}[t]
\caption{Constraint-Aware Subset Selection}
\label{alg:constraint_subset_app}
\begin{algorithmic}[1]
\Require Dataset \(X\), centers \(\{\mu_g\}_{g\in\mathcal{K}}\), cannot-link constraints \(\mathcal{T}_{cl}\), percentile \(p \in [10,30]\)
\Ensure Subset \(S \subseteq \hat{\mathcal{S}}\)
\State Compute baseline assignment \(g^*(i)\) for all \(i \in \hat{\mathcal{S}}\)
\State Initialize \(V \gets \emptyset\)
\For{each point \(i \in \hat{\mathcal{S}}\)}
    \State \(d_{\mathrm{cur}} \gets \|x_i - \mu_{g^*(i)}\|^2\)
    \State \(d_{\mathrm{alt}} \gets \min_{g \neq g^*(i)} \|x_i - \mu_g\|^2\)
    \State \(m(i) \gets d_{\mathrm{cur}} - d_{\mathrm{alt}}\)
\EndFor
\For{each constraint \(\{i,j\} \in \mathcal{T}_{cl}\)}
    \If{\(g^*(i) = g^*(j)\)}
        \State \(V \gets V \cup \{i, j\}\)
    \EndIf
\EndFor
\State Let \(\theta_p\) be the \(p\)-th percentile of \(\{m(i)\}_{i\in\hat{\mathcal{S}}}\)
\State \(\tau \gets \max\{0,\,-\theta_p\}\)
\State \(A \gets \{ i \in \hat{\mathcal{S}} : m(i) > -\tau \}\)
\State \Return \(S \gets V \cup A\)
\end{algorithmic}
\end{algorithm}

\begin{algorithm}[t]
\caption{Information Geometric Subset Selection}
\label{alg:igass_app}
\begin{algorithmic}[1]
\Require Soft assignments \(\{p_i\}_{i\in\hat{\mathcal{S}}}\), violation set \(V\), budget \(m \ge |V|\)
\Ensure Subset \(S\) with \(|S| = m\)
\State \(S \gets V\)
\For{each \(i \in \hat{\mathcal{S}}\setminus V\)}
    \State Let \(g_1, g_2\) be the indices of the two largest entries of \(p_i\) (ties broken deterministically)
    \State \(q_i \gets \Big(\frac{p_i^{(g_1)}}{p_i^{(g_1)}+p_i^{(g_2)}},\ \frac{p_i^{(g_2)}}{p_i^{(g_1)}+p_i^{(g_2)}}\Big)\)
    \State \(d_{\mathrm{FR}}(q_i,b) \gets 2\arccos\!\Big(\frac{1}{\sqrt{2}}(\sqrt{q_{i1}}+\sqrt{q_{i2}})\Big)\), where \(b=(\tfrac12,\tfrac12)\)
    \State \(J_i \gets 1 - \frac{2}{\pi}d_{\mathrm{FR}}(q_i,b)\)
\EndFor
\State Add to \(S\) the \(m-|V|\) indices in \(\hat{\mathcal{S}}\setminus V\) with largest \(J_i\)
\State \Return \(S\)
\end{algorithmic}
\end{algorithm}

\subsection{Information geometric details}
\label{app:ig_details}

This subsection states the PASS-IG budget rule, its computational cost, and the greedy selection property for additive scores.

\subsubsection{Budget rule}

PASS-IG selects a working set of size \(m\) with \(m \ge |V|\), ensuring that all cannot-link violations under the baseline assignment are included. A practical rule is
\[
m \;=\; \min\Big(\alpha|\hat{\mathcal{S}}|,\; |V| + \beta K \log|\hat{\mathcal{S}}|\Big),
\]
with \(\alpha \in [0.1,0.3]\) and \(\beta \in [2,5]\). The cap \(\alpha|\hat{\mathcal{S}}|\) bounds worst-case cost, while the second term includes \(V\) and adds a logarithmic number of high-score points per cluster.

\subsubsection{Complexity}

Computing \(J_i\) requires the two largest entries of \(p_i\), found in \(O(K)\) time per point. Ranking scores over \(\hat{\mathcal{S}}\setminus V\) costs \(O(|\hat{\mathcal{S}}|\log|\hat{\mathcal{S}}|)\). The total complexity is \(O(|\hat{\mathcal{S}}|(K+\log|\hat{\mathcal{S}}|))\).

\subsubsection{Greedy selection for additive scores}

\begin{proposition}[Greedy selection for additive scores]
Let \(F(S)=\sum_{i\in S}J_i\). Among all sets \(S\) satisfying \(|S|=m\) and \(V\subseteq S\),
a maximizer is obtained by taking \(S=V\) and adding the \(m-|V|\) indices in \(\hat{\mathcal{S}}\setminus V\)
with largest scores.
\end{proposition}

\section{Certified Subset Repair for Cannot-Link Constraints}
\label{app:certified_repair}

This appendix formalizes the certified repair view used in Section~\ref{sec:certified_repair}. The setting is a graph of cannot-link constraints and a current assignment. A reveal set \(S\) specifies which vertices may change. A certified repair primitive with verifiable outcomes is defined, an equivalent list coloring formulation is given, a sufficient local slack condition is stated with a polynomial time witness, and a posteriori exactness certificates are provided when lower bounds are tight.

\subsection{Certified repair as a proof producing primitive}

Let \(V\) be the vertex set and \([K]=\{1,\dots,K\}\) be the label set. An assignment is a map
\(a:V\to[K]\). Cannot-link constraints are edges of an undirected graph \(G=(V,E)\).
An edge \(\{i,j\}\in E\) is violated by \(a\) if \(a(i)=a(j)\).

\begin{definition}[Verifier: batch of violated edges]
Given \(a\), define the violated edge set
\[
E_{\mathrm{viol}}(a) \;=\; \{\,\{i,j\}\in E : a(i)=a(j)\,\}.
\]
The verifier returns \(\mathrm{ACCEPT}\) if \(E_{\mathrm{viol}}(a)=\emptyset\) and
\(\mathrm{FAIL}\) otherwise.
\end{definition}

Each vertex \(i\in V\) has a nonnegative reveal cost \(w(i)\ge 0\). A reveal set is \(S\subseteq V\),
and vertices in \(V\setminus S\) are frozen.

\begin{definition}[Movable endpoints repair semantics]
Given \(a\) and \(S\), a repaired assignment \(a'\) may change labels only on \(S\) and must satisfy
\[
a'(j)=a(j)\ \ \forall j\notin S,
\qquad
a'(i)\neq a'(j)\ \ \forall \{i,j\}\in E.
\]
Equivalently, \(a'\) is a proper \(K\)-coloring of \(G\) consistent with frozen labels.
\end{definition}

\begin{definition}[Frozen feasibility and batch cover]
Fix \(a\) and \(S\). The pair \((a,S)\) is frozen feasible if
\[
a(u)\neq a(v)\quad \forall\{u,v\}\in E \ \text{with}\ u,v\notin S.
\]
Equivalently, \(S\) hits every violated edge: for all \(\{i,j\}\in E_{\mathrm{viol}}(a)\),
\(i\in S\) or \(j\in S\).
\end{definition}

\begin{definition}[Minimum repairable reveal cost]
Define
\[
\mathrm{OPT}_{\mathrm{rep}}(a)
\;=\;
\min_{S\subseteq V}\ \sum_{i\in S} w(i)
\quad
\text{subject to}\quad
(a,S)\ \text{is frozen feasible and admits a repair.}
\]
\end{definition}

\begin{definition}[Certified repair primitive]
Fix \(a\), costs \(w(\cdot)\), and an optional budget \(B\).
A certified repair primitive returns one of the following outcomes together with a witness that is
checkable in polynomial time:
\begin{enumerate}
\item \textbf{ACCEPT:} a proof that \(E_{\mathrm{viol}}(a)=\emptyset\).
\item \textbf{REPAIR:} a repaired assignment \(a'\) satisfying the repair semantics, together with
      a repair witness.
\item \textbf{EXACT:} a repaired assignment \(a'\) together with a lower bound certificate proving
      \(\sum_{i\in S} w(i) = \mathrm{OPT}_{\mathrm{rep}}(a)\) for the stated reveal set \(S\).
\item \textbf{OBSTRUCT:} a subinstance \(H\subseteq G\) with an infeasibility certificate showing
      that no repair exists within a stated budget on that core under inherited frozen labels.
\end{enumerate}
\end{definition}

\subsection{List coloring reformulation}

Given frozen vertices \(V\setminus S\) and their labels \(a(\cdot)\), each vertex \(i\in S\) has
allowed labels that avoid frozen neighbor labels:
\[
L_i(a,S) \;=\; [K]\setminus \{\, a(j) : \{i,j\}\in E,\ j\notin S \,\}.
\]

\begin{definition}[Induced lists]
Given \((a,S)\), define the list system \(L(a,S)=\{L_i(a,S)\}_{i\in S}\) on the induced subgraph
\(G[S]\).
\end{definition}

\begin{lemma}[Repairability equals list colorability under frozen feasibility]
If \((a,S)\) is frozen feasible, then \(S\) is repairable for \(a\) if and only if \(G[S]\) admits
a proper list coloring \(\phi:S\to[K]\) such that \(\phi(i)\in L_i(a,S)\) for all \(i\in S\).
\end{lemma}

\begin{proof}
If \(a'\) is a feasible repair, then \(\phi=a'|_S\) is proper on \(G[S]\) and avoids frozen neighbor
labels, so \(\phi(i)\in L_i(a,S)\).
Conversely, if \(\phi\) is a proper list coloring of \(G[S]\) and \((a,S)\) is frozen feasible,
define \(a'(i)=\phi(i)\) for \(i\in S\) and \(a'(j)=a(j)\) for \(j\notin S\). Edges inside \(S\) are
proper by \(\phi\), edges between \(S\) and \(V\setminus S\) are proper by list feasibility, and
edges inside \(V\setminus S\) are proper by frozen feasibility.
\end{proof}

\subsection{Local slack certificates}

A sufficient condition for list colorability is based on degeneracy.

\begin{definition}[Degeneracy and boundary degree]
Let \(\deg(H)\) denote the degeneracy of a graph \(H\). For a reveal set \(S\), define the boundary degree
\[
\Delta_{\mathrm{out}}(S) \;=\; \max_{i\in S} |\{\, j\notin S : \{i,j\}\in E \,\}|.
\]
\end{definition}

\begin{lemma}[Local list size bound]
For any \(i\in S\), \(|L_i(a,S)| \ge K - \Delta_{\mathrm{out}}(S)\).
\end{lemma}

\begin{proof}
Each frozen neighbor forbids at most one label. There are at most
\(\Delta_{\mathrm{out}}(S)\) frozen neighbors.
\end{proof}

\begin{lemma}[Greedy list coloring under degeneracy]
Let \(H\) be \(d\)-degenerate. If every vertex \(v\) has a list \(L(v)\) with \(|L(v)|\ge d+1\), then
\(H\) is list colorable from these lists, and a list coloring can be constructed greedily along a
degeneracy ordering.
\end{lemma}

\begin{assumption}[Instance wise local slack]
\label{assump:local_slack}
Fix \((a,S)\). The set \(S\) satisfies local slack for \(a\) if
\[
\min_{i\in S} |L_i(a,S)| \;\ge\; \deg(G[S]) + 1.
\]
\end{assumption}

\begin{theorem}[Local slack and frozen feasibility imply certified repairability]
If \((a,S)\) is frozen feasible and \(S\) satisfies Assumption~\ref{assump:local_slack}, then \(S\)
is repairable for \(a\). Moreover, a repair can be constructed in polynomial time by greedy list
coloring on \(G[S]\).
\end{theorem}

\begin{proof}
Assumption~\ref{assump:local_slack} implies \(|L_i(a,S)|\ge \deg(G[S])+1\) for all \(i\in S\).
Apply the greedy list coloring lemma to obtain a list coloring \(\phi\) of \(G[S]\). The list coloring
lemma then yields a global repair.
\end{proof}

\paragraph{Witness format.}
A checkable witness consists of: the computed lists \(L_i(a,S)\), a degeneracy peeling trace for
\(G[S]\), and the greedy coloring trace. An external checker can replay the peeling and verify that
each colored vertex chose an available label distinct from its already colored neighbors.

\subsection{Universal lower bounds and a posteriori exactness}

Let \(G_{\mathrm{viol}}=(V,E_{\mathrm{viol}}(a))\) be the violation batch graph.

\begin{definition}[Minimum vertex cover of the violation batch]
Define
\[
\mathrm{OPT}_{\mathrm{VC}}(a)
\;=\;
\min_{C\subseteq V}\ \sum_{i\in C} w(i)
\quad \text{subject to}\quad
\forall\{i,j\}\in E_{\mathrm{viol}}(a),\ i\in C\ \text{or}\ j\in C.
\]
\end{definition}

\begin{lemma}[Universal lower bound]
For any \(a\), \(\mathrm{OPT}_{\mathrm{rep}}(a)\ge \mathrm{OPT}_{\mathrm{VC}}(a)\).
\end{lemma}

\begin{proof}
Any feasible reveal set must hit every violated edge, hence is a vertex cover of \(G_{\mathrm{viol}}\).
\end{proof}

\begin{definition}[VC tight exactness certificate]
A VC tight certificate for \(a\) is a triple \((C,\mathrm{Rep}(C),\pi)\) where:
\begin{enumerate}
\item \(C\subseteq V\) is a vertex cover of \(G_{\mathrm{viol}}\);
\item \(\mathrm{Rep}(C)\) is a checkable repair witness for \(C\), for example an explicit list coloring
      of \(G[C]\) under \(L(a,C)\), or a local slack proof with a greedy trace;
\item \(\pi\) is a checkable proof that \(C\) is an optimal weighted vertex cover for \(G_{\mathrm{viol}}\).
\end{enumerate}
\end{definition}

\begin{theorem}[A posteriori exactness in the VC tight regime]
If a VC tight exactness certificate \((C,\mathrm{Rep}(C),\pi)\) exists for \(a\), then
\[
\mathrm{OPT}_{\mathrm{rep}}(a) \;=\; \mathrm{OPT}_{\mathrm{VC}}(a) \;=\; \sum_{i\in C} w(i),
\]
and a globally feasible repair is constructible from \(\mathrm{Rep}(C)\).
\end{theorem}

\begin{proof}
The universal lower bound gives \(\mathrm{OPT}_{\mathrm{rep}}(a)\ge \mathrm{OPT}_{\mathrm{VC}}(a)\).
The repair witness \(\mathrm{Rep}(C)\) yields a repair on \(C\), so
\(\mathrm{OPT}_{\mathrm{rep}}(a)\le \sum_{i\in C} w(i)\). The certificate \(\pi\) proves
\(\sum_{i\in C} w(i)=\mathrm{OPT}_{\mathrm{VC}}(a)\), hence equality holds.
\end{proof}

\subsection{Repair ILP and LP tight certificates}

A compact ILP is given whose optimal value equals \(\mathrm{OPT}_{\mathrm{rep}}(a)\). Introduce variables
\(s_i\in\{0,1\}\) indicating whether \(i\) is revealed and variables \(y_{i,c}\in\{0,1\}\) indicating
the repaired label.

\begin{definition}[Repair ILP]
\label{def:repair_ilp}
Fix \(a\) on \(V\) and graph \(G=(V,E)\). Consider
\begin{subequations}
\label{eq:repair_ilp}
\begin{align}
\min \quad & \sum_{i\in V} w(i)s_i \\
\text{s.t.}\quad
& \sum_{c=1}^K y_{i,c} = 1 \qquad && \forall i\in V, \label{eq:repair_ilp:one}\\
& y_{i,c} \le s_i \qquad && \forall i\in V,\ \forall c\neq a(i), \label{eq:repair_ilp:freeze1}\\
& y_{i,a(i)} \ge 1 - s_i \qquad && \forall i\in V, \label{eq:repair_ilp:freeze2}\\
& y_{i,c} + y_{j,c} \le 1 \qquad && \forall\{i,j\}\in E,\ \forall c\in [K], \label{eq:repair_ilp:cl}\\
& s_i \in \{0,1\},\ y_{i,c}\in\{0,1\} \qquad && \forall i\in V,\ \forall c\in[K]. \label{eq:repair_ilp:int}
\end{align}
\end{subequations}
Let \(\mathrm{OPT}_{\mathrm{ILP}}(a)\) be its optimal value.
\end{definition}

\begin{lemma}[Equivalence of repair cost and ILP cost]
For every \(a\), \(\mathrm{OPT}_{\mathrm{rep}}(a)=\mathrm{OPT}_{\mathrm{ILP}}(a)\).
\end{lemma}

\begin{proof}
Given a repair with reveal set \(S\) and repaired assignment \(a'\), set \(s_i=\mathbb{I}[i\in S]\)
and \(y_{i,c}=\mathbb{I}[a'(i)=c]\). The constraints enforce one label per vertex, freezing outside
\(S\), and proper coloring on every edge. The objective equals \(\sum_{i\in S}w(i)\).
Conversely, from any feasible \((s,y)\), define \(S=\{i:s_i=1\}\) and \(a'(i)=c\) where \(y_{i,c}=1\).
Constraints imply \(a'(i)=a(i)\) for \(i\notin S\) and proper coloring on \(E\). The objective is
\(\sum_{i\in S}w(i)\). Optimizing both directions yields equality.
\end{proof}

\begin{definition}[LP relaxation]
Let \(\mathrm{OPT}_{\mathrm{LP}}(a)\) be the optimal value of the relaxation of
\eqref{eq:repair_ilp} obtained by replacing \eqref{eq:repair_ilp:int} with
\(0\le s_i \le 1\) and \(0\le y_{i,c}\le 1\).
\end{definition}

\begin{lemma}[Hierarchy of bounds]
For any \(a\),
\[
\mathrm{OPT}_{\mathrm{VC}}(a) \;\le\; \mathrm{OPT}_{\mathrm{LP}}(a) \;\le\;
\mathrm{OPT}_{\mathrm{ILP}}(a) \;=\; \mathrm{OPT}_{\mathrm{rep}}(a).
\]
\end{lemma}

\begin{proof}
\(\mathrm{OPT}_{\mathrm{LP}}(a)\le \mathrm{OPT}_{\mathrm{ILP}}(a)\) because the LP is a relaxation.
For \(\mathrm{OPT}_{\mathrm{VC}}(a)\le \mathrm{OPT}_{\mathrm{LP}}(a)\), let \(\{i,j\}\in E_{\mathrm{viol}}(a)\) and
\(c_0=a(i)=a(j)\). From \eqref{eq:repair_ilp:freeze2}, \(y_{i,c_0}\ge 1-s_i\) and \(y_{j,c_0}\ge 1-s_j\).
From \eqref{eq:repair_ilp:cl}, \(y_{i,c_0}+y_{j,c_0}\le 1\). Hence \(s_i+s_j\ge 1\), which are the
fractional vertex cover constraints for the violated batch.
\end{proof}

\begin{definition}[LP dual certificate]
An LP dual certificate is a feasible dual solution whose objective value equals a claimed bound
\(L\). The certificate is checkable by verifying dual feasibility and recomputing the dual objective.
\end{definition}

\begin{theorem}[LP tight a posteriori exactness]
Fix \(a\). Suppose:
\begin{enumerate}
\item a feasible repair of cost \(U\), and
\item an LP dual certificate proving \(\mathrm{OPT}_{\mathrm{LP}}(a)=L\).
\end{enumerate}
If \(U=L\), then \(U=\mathrm{OPT}_{\mathrm{rep}}(a)\) and this optimality is certified.
\end{theorem}

\begin{proof}
By the hierarchy of bounds, \(\mathrm{OPT}_{\mathrm{LP}}(a)\le \mathrm{OPT}_{\mathrm{rep}}(a)\le U\).
The dual certificate proves \(\mathrm{OPT}_{\mathrm{LP}}(a)=L\), hence \(L\le \mathrm{OPT}_{\mathrm{rep}}(a)\le U\).
If \(U=L\), then all inequalities are equalities.
\end{proof}

\subsection{Certified decomposition when local slack fails}

Local slack is sufficient for repairability. When it fails, a checkable slack-failure core is extracted.

\begin{definition}[Peeling extractor and slack-failure core]
Fix \((a,S)\) and compute lists \(L(a,S)\). Let
\[
\ell \;=\; \min_{i\in S} |L_i(a,S)|.
\]
Define a peeling process on \(H_0=G[S]\): repeatedly delete any vertex of the current graph of
degree at most \(\ell-1\), producing \(H_{t+1}\), until no such vertex remains. Let \(H_{\ell}(a,S)\)
be the final remaining graph, possibly empty.
\end{definition}

\begin{lemma}[Certified slack or core dichotomy]
For any \((a,S)\), exactly one of the following holds:
\begin{enumerate}
\item \(H_{\ell}(a,S)=\emptyset\). Then \(\deg(G[S])\le \ell-1\), so local slack holds and \(S\) is repairable.
\item \(H_{\ell}(a,S)\neq\emptyset\). Then \(H_{\ell}(a,S)\) has minimum degree at least \(\ell\), hence
\(\deg(G[S])\ge \ell\) and local slack fails. The condition \(\delta(H_{\ell})\ge \ell\) is checkable by
direct inspection.
\end{enumerate}
\end{lemma}

\begin{proof}
A graph is \((\ell-1)\)-degenerate if and only if iterative deletion of vertices of degree at most
\(\ell-1\) deletes all vertices. The two cases follow from the peeling process definition.
\end{proof}

\subsection{Obstruction cores and certified failure on structured subgraphs}

\begin{definition}[Budgeted repair decision]
Given \(a\) and budget \(B\ge 0\), decide whether there exists a reveal set \(S\) and repair \(a'\)
with \(\sum_{i\in S}w(i)\le B\) satisfying the repair semantics.
\end{definition}

\begin{definition}[Obstruction core]
An obstruction core for \((a,B)\) is a subgraph \(H\subseteq G\) with vertex set \(U\), together with
a checkable witness showing that no repair exists on \(U\) with reveal cost at most \(B_U\) for some
explicit \(B_U\le B\), under the frozen colors on \(V\setminus U\). A valid witness can be produced by
a rerunnable exact backend on the restricted instance, for example dynamic programming tables on a
tree decomposition.
\end{definition}

\begin{definition}[Treewidth parameter]
Let \(\mathrm{tw}(H)\) denote the treewidth of \(H\).
\end{definition}

\begin{theorem}[Exact dynamic programming on bounded treewidth cores]
Given a tree decomposition of a core \(H\) of width \(t=\mathrm{tw}(H)\), the minimum repairable
reveal problem restricted to \(H\) can be solved exactly in time
\[
O\!\left(|V(H)|\cdot (K+2)^{t+1}\right),
\]
up to polynomial factors in \(t\). The algorithm outputs either an optimal restricted repair or a
checkable infeasibility witness for a stated budget on \(H\), which constitutes an obstruction core
certificate.
\end{theorem}

\subsection{Empirical certificate coverage}
\label{app:empirical_certificate_coverage}

This section reports how often the local slack certificate in Theorem~\ref{thm:local_slack_main} applies when the repair routine is invoked. Since the theorem gives a sufficient condition, certification and repair success may differ. Table~\ref{tab:certificate_coverage_app} summarizes how often repair is triggered, the fraction of certificate required calls that are certified, and repair success rates with runtime overhead. Under this configuration, PASS-IG reaches feasibility without invoking repair, so certificate coverage is not applicable for PASS-IG.

\begin{table}[h]
\centering
\caption{Coverage of the local slack repair certificate. Cert rate is computed over certificate required repair calls. Repair ok is computed over all repair calls. Median ms per call is computed per run and aggregated over runs with at least one repair call. Under this configuration, PASS-IG does not invoke repair, so coverage is not applicable.}
\label{tab:certificate_coverage_app}
\begin{tabular}{lrrrrrrrr}
\toprule
Method & Runs & Runs w/ repair & Calls & Cert req & Certified & Cert rate & Repair ok & Median ms per call \\
\midrule
PASS-CA & 21 & 13 & 23 & 13 & 3 & 23.1\% & 26.1\% & 38.5 \\
PASS-IG & 21 & 0  & 0  & 0  & 0 & n/a & n/a & n/a \\
\bottomrule
\end{tabular}
\end{table}
\section{0--1 ILP for Phase 3}
\label{sec:0_1_ILP}

Let $S \subseteq \hat{\mathcal{S}}$ be the working subset and let $\{\mu_g\}_{g=1}^K$ be the current centroids.
For each $i \in S$, let $K_i \subseteq \{1,\dots,K\}$ be the candidate label set and introduce binary variables
$z_{i,g}\in\{0,1\}$ for $g\in K_i$, indicating whether $i$ is assigned to cluster $g$.

Let $g_i^{\mathrm{cur}}$ denote the current label and let $c_i^{\mathrm{cur}}=\lVert x_i-\mu_{g_i^{\mathrm{cur}}}\rVert^2$.
Define the incremental cost
\[
\Delta_{i,g}\;=\;\lVert x_i-\mu_g\rVert^2\;-\;c_i^{\mathrm{cur}},\qquad g\in K_i.
\]
If $g_i^{\mathrm{cur}}\in K_i$, then $\Delta_{i,g_i^{\mathrm{cur}}}=0$. Since $\sum_{g\in K_i} z_{i,g}=1$ (below),
the term $\sum_{i\in S}c_i^{\mathrm{cur}}$ is constant and Phase~3 minimizes a restricted SSE surrogate via
\[
\min_{z}\;\sum_{i\in S}\sum_{g\in K_i}\Delta_{i,g}\,z_{i,g}.
\]

\subsection{Feasibility Constraints}

\paragraph{One-hot assignment.}
Each $i\in S$ selects exactly one label:
\[
\sum_{g\in K_i} z_{i,g}=1\,,\qquad \forall i\in S.
\]

\paragraph{Cannot-link within $S$.}
For any $(u,v)\in\mathcal T_{\mathrm{cl}}$ with $u,v\in S$,
\[
z_{u,g}+z_{v,g}\le 1\,,\qquad \forall g\in K_u\cap K_v.
\]

\paragraph{Cross-$S$ consistency (frozen neighbors).}
For $(u,v)\in\mathcal T_{\mathrm{cl}}$ with $u\in S$ and $v\notin S$, the vertex $v$ is frozen to its current label
$g_v^{\mathrm{cur}}$. Therefore $u$ cannot take label $g_v^{\mathrm{cur}}$, enforced by
\[
z_{u,g_v^{\mathrm{cur}}}=0 \quad \text{whenever } g_v^{\mathrm{cur}}\in K_u.
\]
In implementation, $K_u$ is constructed to include current labels of CL neighbors (including neighbors outside $S$), so forbidden frozen-neighbor labels are present as variables and can be clamped to zero; equivalently, these labels may be removed from $K_u$ before creating variables.
If after forbidding frozen-neighbor labels $K_u=\emptyset$ for some $u$, then the restricted subproblem is infeasible; in that case a deterministic fallback described in the main text is used to produce an admissible update.

\subsection{Implementation Details}

\noindent\textbf{Candidate sets.}
For $i\in S$, initialize $K_i$ with the nearest $G$ centroids to $x_i$ (default $G{=}4$). If $i$ is a CL violator under the current assignment
or adjacent (in the CL graph) to a violator, set $K_i\leftarrow\{1,\dots,K\}$.
Otherwise, augment $K_i$ with the current label of $i$ and the current labels of its CL neighbors (including neighbors outside $S$).

\noindent\textbf{Warm start and time limit.}
If $g_i^{\mathrm{cur}}\in K_i$, set $z_{i,g_i^{\mathrm{cur}}}\leftarrow 1$ as a MIP start and solve under a wall-clock time limit.
If a MILP solver is unavailable, use a deterministic feasible local-search fallback.

\begin{algorithm}[H]
\caption{Phase 3: Restricted 0--1 ILP on $S$}
\begin{algorithmic}[1]
\Require Data $X$, centers $\{\mu_g\}$, current labels $\{g_i^{\mathrm{cur}}\}$, subset $S$, CL set $\mathcal T_{\mathrm{cl}}$, parameter $G$
\State Build CL adjacency; mark violators $\mathcal V=\{i:\exists (i,j)\in\mathcal T_{\mathrm{cl}},\ g_i^{\mathrm{cur}}=g_j^{\mathrm{cur}}\}$
\For{$i\in S$}
  \State $K_i \gets$ top-$G$ nearest centroids to $x_i$
  \If{$i\in\mathcal V$ \textbf{ or } $i$ is adjacent to $\mathcal V$}
    \State $K_i \gets \{1,\dots,K\}$
  \EndIf
  \State $K_i \gets K_i \cup \{g_i^{\mathrm{cur}}\} \cup \{g_j^{\mathrm{cur}}:(i,j)\in\mathcal T_{\mathrm{cl}}\}$
\EndFor
\State For $i\in S$, set $c_i^{\mathrm{cur}}=\lVert x_i-\mu_{g_i^{\mathrm{cur}}}\rVert^2$ and $\Delta_{i,g}=\lVert x_i-\mu_g\rVert^2-c_i^{\mathrm{cur}}$ for $g\in K_i$
\State Create binaries $z_{i,g}$ for $i\in S$, $g\in K_i$
\State Add constraints: (i) $\sum_{g\in K_i} z_{i,g}=1$; (ii) $z_{u,g}+z_{v,g}\le 1$ for $(u,v)\in\mathcal T_{\mathrm{cl}}$ with $u,v\in S$ and $g\in K_u\cap K_v$; (iii) set $\mathrm{UB}(z_{u,g_v^{\mathrm{cur}}})=0$ for $(u,v)\in\mathcal T_{\mathrm{cl}}$ with $u\in S$, $v\notin S$
\State Set objective $\min \sum_{i\in S}\sum_{g\in K_i}\Delta_{i,g}\,z_{i,g}$
\State Warm start: if $g_i^{\mathrm{cur}}\in K_i$ then set $z_{i,g_i^{\mathrm{cur}}}\gets 1$; solve with time limit; decode $\hat g_i \gets \arg\max_{g\in K_i} z_{i,g}$ for $i\in S$
\end{algorithmic}
\end{algorithm}

\paragraph{Size and complexity.}
The model has $\sum_{i\in S}|K_i|$ binaries, thus $O(|S|G)$ under normal operation, rising to $O(|S|K)$ only for nodes in or adjacent to CL violations due to candidate expansion.
Constraints include $|S|$ one-hot equalities; at most
\[
\sum_{(u,v)\in \mathcal T_{\mathrm{cl}}\cap (S\times S)} |K_u\cap K_v|
\]
CL inequalities (added only when $K_u\cap K_v\neq\varnothing$); and a linear number of cross-$S$ clamps from $\mathcal T_{\mathrm{cl}}$.

\section{Runtime and Computational Performance}
\label{sec:comp_analysis}

\subsection{Phase-Wise Cost Breakdown}
Iteration time is split into three phases: subset selection, restricted 0--1 ILP reassignment, and centroid recomputation. Under \textsc{Both} with IG, reassignment typically accounts for the largest share, while selection and centroid updates contribute comparable amounts (Table~\ref{tab:phase_breakdown_ig_both}). These shares vary with dataset size. In \texttt{iris} and \texttt{seeds}, reassignment dominates. In \texttt{skin}, selection and centroid updates account for a larger portion of time.

Two factors explain this behavior. First, the reassignment problem size per iteration, measured by the maximum number of binary variables, grows with $n$ (Figure~\ref{fig:ilp_size_vs_n_ig_both}). Second, optimization is restricted to an active subset whose mean fraction of $n$ decreases as $n$ increases. In \texttt{skin}, the active subset is below one percent of $n$, whereas in \texttt{iris} and \texttt{seeds} it is near one sixth. As a result, the integer program grows in absolute size with $n$, but its relative time share is smaller in the largest instance. Selection and centroid updates touch all points each iteration, which increases their time shares in \texttt{skin}.

\begin{table}[h]
\centering
\small
\begin{tabular}{lccccc}
\toprule
Dataset & $n$ & Selection & ILP & Centroids & Max Binaries \\
\midrule
iris     & 150      & 12.6\% & 68.4\% & 8.5\%  & 56 \\
seeds    & 210      & 10.9\% & 71.7\% & 7.6\%  & 125 \\
hemi     & 1{,}955  & 15.4\% & 57.8\% & 13.7\% & 1{,}441 \\
pr2392   & 2{,}392  & 16.3\% & 56.3\% & 14.8\% & 1{,}574 \\
AC\_FL   & 7{,}195  & 16.4\% & 58.8\% & 15.3\% & 4{,}756 \\
rds\_cnt & 10{,}000 & 17.3\% & 56.2\% & 17.1\% & 6{,}647 \\
HTRU2\_L & 17{,}898 & 18.4\% & 51.1\% & 17.8\% & 10{,}635 \\
skin     & 245{,}057 & 35.0\% & 25.6\% & 25.9\% & 49{,}507 \\
\bottomrule
\end{tabular}
\caption{Phase-wise time shares (sum over iterations) and the maximum number of binaries per iteration for IG under \textsc{Both}.}
\label{tab:phase_breakdown_ig_both}
\end{table}

The instrumented phases account for most of the total time. The remainder corresponds to orchestration and utilities, including data structures, input output, and logging, so the columns do not sum to one. For small and medium instances, improvements should prioritize reassignment. As $n$ increases, optimizations should target selection and centroid updates through sampling schemes or partial updates.

\begin{figure}[H]
\centering
\includegraphics[width=0.9\linewidth]{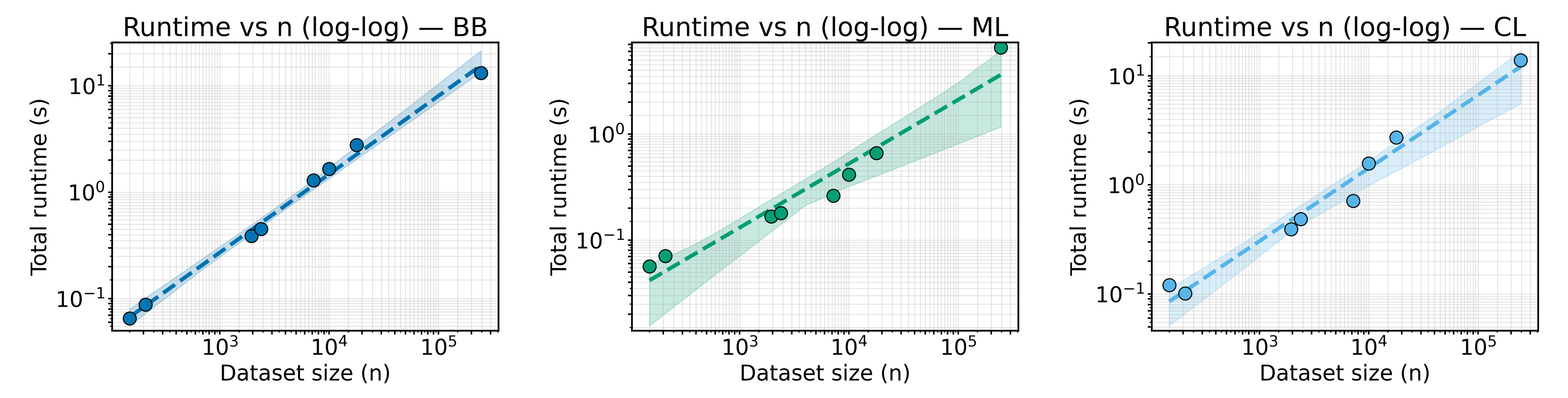}
\caption{Total runtime vs.\ $n$ (log--log). Runtime grows by a factor of $201$ for a factor of $1{,}633$ increase in dataset size.}
\label{fig:runtime_vs_n_ig_both}
\end{figure}

\begin{figure}[H]
\centering
\includegraphics[width=0.9\linewidth]{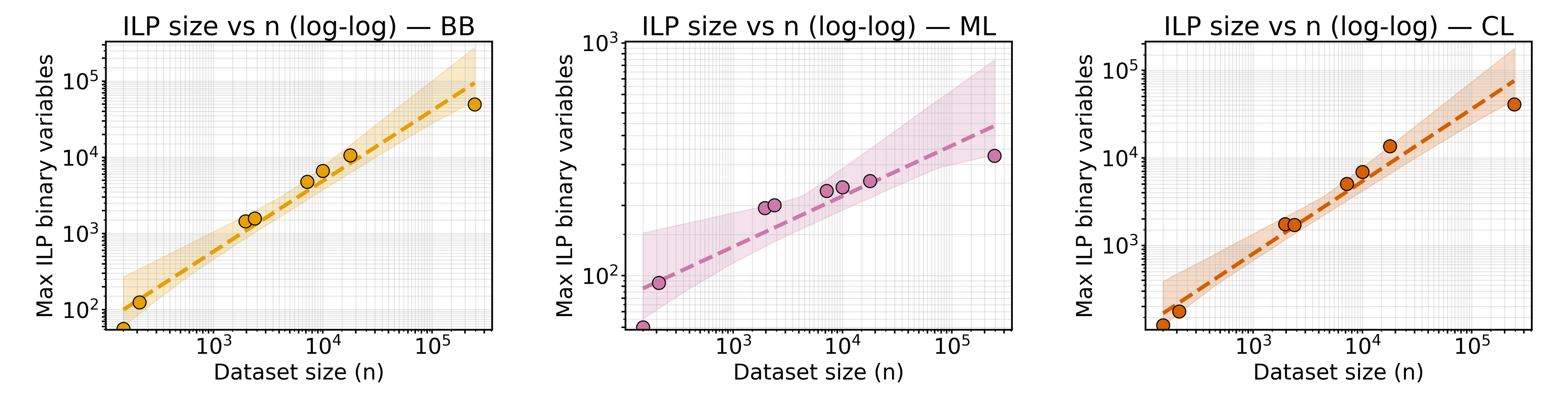}
\caption{Maximum ILP binaries vs.\ $n$ (log--log). Across the range, the maximum binaries increase by a factor of $884$.}
\label{fig:ilp_size_vs_n_ig_both}
\end{figure}

\begin{table}[h]
\centering
\small
\setlength{\tabcolsep}{5pt}
\begin{tabular}{l l r r r r r r}
\toprule
Dataset & Type & $|S|/n$ (IG) & $|S|/n$ (CA) & $|V|/n$ (IG) & $|V|/n$ (CA) & Viol (IG) & Viol (CA) \\
\midrule
HTRU2\_L & BB & 0.136 & 0.187 & 0.078 & 0.093 & 0 & 18 \\
HTRU2\_L & CL & 0.118 & 0.185 & 0.071 & 0.101 & 0 & 22 \\
HTRU2\_L & ML & 0.000 & 0.000 & 0.000 & 0.000 & 0 & 0 \\
hemi & BB & 0.187 & 0.327 & 0.140 & 0.269 & 0 & 11 \\
hemi & CL & 0.148 & 0.168 & 0.114 & 0.136 & 0 & 0 \\
hemi & ML & 0.000 & 0.000 & 0.000 & 0.000 & 0 & 0 \\
iris & BB & 0.097 & 0.097 & 0.000 & 0.000 & 0 & 0 \\
iris & CL & 0.100 & 0.100 & 0.000 & 0.053 & 0 & 0 \\
iris & ML & 0.000 & 0.000 & 0.000 & 0.000 & 0 & 0 \\
pr2392 & BB & 0.149 & 0.154 & 0.092 & 0.096 & 0 & 16 \\
pr2392 & CL & 0.115 & 0.145 & 0.080 & 0.111 & 0 & 3 \\
pr2392 & ML & 0.000 & 0.000 & 0.000 & 0.000 & 0 & 0 \\
rds\_cnt & BB & 0.153 & 0.155 & 0.111 & 0.100 & 0 & 66 \\
rds\_cnt & CL & 0.113 & 0.219 & 0.089 & 0.178 & 0 & 9 \\
rds\_cnt & ML & 0.000 & 0.000 & 0.000 & 0.000 & 0 & 0 \\
seeds & BB & 0.095 & 0.095 & 0.000 & 0.038 & 0 & 0 \\
seeds & CL & 0.150 & 0.100 & 0.014 & 0.029 & 0 & 0 \\
seeds & ML & 0.000 & 0.000 & 0.000 & 0.000 & 0 & 0 \\
skin & BB & 0.047 & 0.092 & 0.028 & 0.055 & 0 & 0 \\
skin & CL & 0.027 & 0.036 & 0.017 & 0.022 & 0 & 0 \\
skin & ML & 0.000 & 0.000 & 0.000 & 0.000 & 0 & 0 \\
\bottomrule
\end{tabular}
\caption{Working-set statistics and feasibility summary for PASS-IG and PASS-CA. Values are medians over iterations within a run. Rows for ML-only are included for completeness; in this setting $V=\emptyset$ and $S=\emptyset$ by construction.}
\label{tab:subset_summary_ig_ca_app}
\end{table}

\subsection{Working-Set Statistics: PASS-IG vs.\ PASS-CA}
\label{app:ig_vs_ca_workingset}

This section compares the working-set behavior of PASS-IG and PASS-CA under the same budget. The median working-set fraction $|S|/n$ and the median violation-set fraction $|V|/n$ are reported together with the final number of violated constraints. PASS-IG reaches feasibility across the reported instances while selecting smaller working sets, whereas PASS-CA may leave residual violations in several \textsc{Both} settings.

\subsection{Time and Memory Scaling with Problem Size}
Total runtime, computed as the sum of per-iteration times, follows a power law in $n$ (Figure~\ref{fig:runtime_vs_n_ig_both}), with $R^2 = 0.994$. From \texttt{iris} ($n = 150$) to \texttt{skin} ($n = 245{,}057$), runtime increases from $0.065$ s to $13.179$ s, a factor of $200$ for a factor of $1{,}600$ in dataset size. The maximum number of binary variables in the per-iteration integer program also scales with $n$ and fits a power law with $R^2 = 0.969$. Over the same range, this count increases by a factor of $884$ (Figure~\ref{fig:ilp_size_vs_n_ig_both}).

This behavior matches the phase-wise profile. The 0--1 ILP is solved on an active subset whose mean fraction of $n$ decreases with $n$. In \texttt{skin}, the active subset is below one percent of $n$, whereas in \texttt{iris} and \texttt{seeds} it is near one sixth. Consequently, the ILP grows in absolute size with $n$, but its time share is smaller in the largest instance. Selection and centroid updates process all points each iteration, so their shares increase with $n$. Overall, runtime grows sublinearly in $n$, while the memory proxy scales more closely with $n$.

\section{q-PCKMeans: A Quantum Pairwise Constrained Clustering Algorithm}

\subsection{Quantum Variational Framework} \label{sec:qvf}

\paragraph{QUBO formulation}
To enable scalable quantum optimization, the problem is restricted to a subset \(S\subseteq\mathcal{S}\) of ambiguous or constraint-violating samples, with \(|S| \ll |\mathcal{S}|\), yielding a reduced 0--1 quadratic program with one-hot, ML, and CL constraints:
\begin{subequations}\label{eqn:qubo}
\begin{align}
\min_{d}\quad
& \sum_{i\in S}\sum_{g\in\mathcal{K}}\Delta_{i,g}\,d_{i,g} \\
\text{s.t.}\quad
& \sum_{g\in\mathcal{K}}d_{i,g}=1,
  &&\forall\,i\in S,\label{eqn:qubo:onehot}\\
& d_{u,g}=d_{v,g},
  &&\forall\,(u,v)\in\mathcal{T}_{ml}\cap(S\times S),\;\forall\,g\in\mathcal{K},\label{eqn:qubo:ml}\\
& d_{u,g}+d_{v,g}\le1,
  &&\forall\,(u,v)\in\mathcal{T}_{cl}\cap(S\times S),\;\forall\,g\in\mathcal{K},\label{eqn:qubo:cl}\\
& d_{i,g}\in\{0,1\},
  &&\forall\,i\in S,\;g\in\mathcal{K}.
\end{align}
\end{subequations}
For edges with one endpoint in \(S\) and the other in \(\mathcal{S}\setminus S\), linear cross-edge penalties forbid assigning the in-\(S\) endpoint to the fixed label of its partner, mirroring the reduced classical formulation. Concretely, for \((u,v)\in\mathcal{T}_{cl}\) with \(u\in S\), \(v\notin S\) and fixed label \(g^{\star}(v)\), add \(\lambda\, d_{u,g^{\star}(v)}\). For \((u,v)\in\mathcal{T}_{ml}\) with \(u\in S\), \(v\notin S\), add
\[
\lambda\!\sum_{g\in\mathcal K}\! d_{u,g}\;-\;2\lambda\, d_{u,g^{\star}(v)},
\]
which, under the one-hot constraint \(\sum_{g} d_{u,g}=1\), is equivalent up to an additive constant to the penalty \(\lambda\,(1-2\,d_{u,g^{\star}(v)})\) that encourages \(u\) to match \(g^{\star}(v)\). The symmetric cases with \(v\in S\), \(u\notin S\) are handled analogously.

The coefficients \(\Delta_{i,g}\) capture the incremental sum of squared errors:
\(
  \Delta_{i,g}
  =
  \|x_i-\mu_g\|^2
  -
  \|x_i-\mu_{g^*(i)}\|^2,
\)
where \(g^*(i)\) denotes the current cluster assignment of sample \(i\).
The binary decision variable \(d_{i,g}=1\) indicates reassigning \(i\) to cluster \(g\).

Problem~\eqref{eqn:qubo} is converted to a QUBO suitable for quantum optimization by embedding the hard constraints as penalty terms in a Hamiltonian. Let
\(
 \lambda \;=\; \sum_{i\in S}\sum_{g\in\mathcal K}\bigl|\Delta_{i,g}\bigr| \;+\;\varepsilon,\ \varepsilon>0,
\)
which is strictly larger than the maximum possible variation of the linear cost over all assignments. Any infeasible solution therefore incurs an energy penalty of at least \(\lambda\), guaranteeing it ranks above every feasible configuration, as shown in Proposition~\ref{prop:lambda}. With this choice of \(\lambda\), one-hot, must-link, and cannot-link constraints translate into the following quadratic penalties:
\begin{subequations}
\begin{align}
H_{\text{one-hot}}
&= \lambda \sum_{i\in S}
    \Bigl(1 - \sum_{g\in\mathcal{K}} d_{i,g}\Bigr)^2,\\
H_{\text{ML}}
&= \lambda \sum_{(u,v)\in\mathcal{T}_{ml}\cap(S\times S)}
    \sum_{g\in\mathcal{K}} (d_{u,g}-d_{v,g})^2,\\
H_{\text{CL}}
&= \lambda \sum_{(u,v)\in\mathcal{T}_{cl}\cap(S\times S)}
    \sum_{g\in\mathcal{K}} d_{u,g}\,d_{v,g}.
\end{align}
\end{subequations}

\noindent The full QUBO Hamiltonian is
\begin{equation}\label{eqn:hamiltonian_2}
H(d)
=
\sum_{i\in S}\sum_{g\in\mathcal{K}}\Delta_{i,g}\,d_{i,g}
\;+\;
H_{\text{one-hot}} + H_{\text{ML}} + H_{\text{CL}}
\;+\; H_{\text{cross}},
\end{equation}
where \(H_{\text{cross}}\) collects the linear cross-edge penalties described above. Up to an additive constant, the Hamiltonian can be written in matrix form as
\(
H(d) = d^{\top} Q\, d \;+\; c^{\top} d.
\)
This formulation yields \(n_{\text{qubo}} = |S|\times K\) binary variables and encodes one-hot, ML, and CL constraints through penalty terms, with a one-hot preserving mixer that remains within the one-hot subspace. Leveraging this invariant subspace, an additive improvement bound is derived for depth \(p=1\) on the expected performance of the algorithm.

A variational algorithm is introduced that operates on a reduced subspace of ambiguous samples while encoding hard constraints via penalties. The framework has three parts: a one-hot preserving mixer that restricts evolution to the feasible one-hot subspace; an additive improvement bound that applies at depth \(p=1\); and a warm-start depth reduction showing how a classical seed reduces circuit depth in practice.

\paragraph{One-Hot Preserving Mixer}
\label{ssec:mixer}
To guarantee that the one-hot constraint remains satisfied throughout the variational evolution, a point-wise XY mixer is used that acts within the one-hot subspace. Must-link and cannot-link constraints are encoded via the penalty Hamiltonian described in Section~\ref{sec:qvf}.
Although intermediate QAOA superpositions can carry amplitude on must-link and cannot-link violating basis states, the penalty terms shift those states upward in energy according to \(\lambda\).
Throughout, let \(X\), \(Y\), and \(Z\) denote the single-qubit Pauli matrices:
\[
  X = \begin{pmatrix}0 & 1 \\ 1 & 0\end{pmatrix}, \quad
  Y = \begin{pmatrix}0 & -i \\ i & 0\end{pmatrix}, \quad
  Z = \begin{pmatrix}1 & 0 \\ 0 & -1\end{pmatrix}.
\]
For a qubit indexed by \(\ell(i,g) = K(i-1) + g\), write \(X_{i,g}\) for \(X\) acting on that qubit, and likewise \(Y_{i,g}\) and \(Z_{i,g}\).

\begin{definition}[Point-wise XY Mixer]
\label{def:xy_mixer}
For each ambiguous sample \(i\in S\) and each unordered pair of clusters
\((g,h)\) with \(g<h\), define the two-qubit unitary
\(
  U_{i,(g,h)}(\beta)
  =
  \exp\!\bigl(-\,i\beta\,[X_{i,g}X_{i,h}+Y_{i,g}Y_{i,h}]\bigr),
\)
where \(X_{i,g},Y_{i,g}\) act on qubit \(\ell(i,g)=K(i-1)+g\).
The global mixer is the product
\(
  U_M(\beta)=\prod_{i\in S}\prod_{g<h} U_{i,(g,h)}(\beta).
\)
\end{definition}

\begin{figure}[ht]
  \centering
  \begin{quantikz}[row sep=0.2cm, column sep=0.7cm]
    \lstick{$\ket{d_{i,1}}$} & \qw & \gate[wires=2]{e^{-i\beta(X\otimes X + Y\otimes Y)}} & \qw & \gate[wires=2]{e^{-i\beta(X\otimes X + Y\otimes Y)}} & \qw \\
    \lstick{$\ket{d_{i,2}}$} & \qw & \qw & \qw & \qw & \qw \\
    \lstick{$\ket{d_{i,3}}$} & \qw & \qw & \gate[wires=2]{e^{-i\beta(X\otimes X + Y\otimes Y)}} & \qw & \qw
  \end{quantikz}
  \caption{One-hot preserving mixer for \(K=3\): apply the three two-qubit blocks \(e^{-i\beta(X_{i,g}X_{i,h}+Y_{i,g}Y_{i,h})}\) for the pairs \((g,h)=(1,2),(1,3),(2,3)\). No single-qubit \(R_X\) rotations are used.}
  \label{fig:xy-mixer-full}
\end{figure}

\begin{lemma}[Invariance of the One-Hot Subspace]
\label{lem:invariance_onehot}
Let
\(
  \mathcal{H}_{\mathrm{one\text{-}hot}}
  =
  \mathrm{Span}\bigl\{\ket{d} :
     \sum_{g=1}^{K} d_{i,g} = 1,\ \forall\,i\in S\bigr\}.
\)
Then, for all \(\beta\),
\(
  U_M(\beta)\,\mathcal{H}_{\mathrm{one\text{-}hot}}
  \subseteq
  \mathcal{H}_{\mathrm{one\text{-}hot}}.
\)
\end{lemma}

\begin{proof}
Fix a sample \(i\) and define the population operator
\(
  N_i = \sum_{g=1}^{K} \frac{1 - Z_{i,g}}{2},
\)
whose eigenvalue on a computational basis state is the Hamming weight \(\sum_g d_{i,g}\).
A state \(\ket{\psi}\) satisfies the one-hot constraint if and only if \(N_i\ket{\psi} = \ket{\psi}\).
Using the commutation relations
\([X,Z] = 2iY\) and \([Y,Z] = -2iX\), one checks
\(
  [N_i,\;X_{i,g}X_{i,h} + Y_{i,g}Y_{i,h}]
  = 0
  \text{ for all }g<h.
\)
Hence each local mixer \(U_{i,(g,h)}(\beta)\) commutes with \(N_i\), and so does the global \(U_M(\beta)\).
Therefore \(U_M(\beta)\) preserves the eigenspace \(N_i = 1\) for every \(i\), and hence preserves \(\mathcal{H}_{\mathrm{one\text{-}hot}}\).
\end{proof}

\begin{remark}
The mixer preserves one-hot but does not enforce must-link or cannot-link relations.
Any intermediate state that violates these relations remains in the computational basis yet incurs an energy penalty of at least \(\lambda>\sum_{i,g}|\Delta_{i,g}|\) as in Proposition~\ref{prop:lambda}.
\end{remark}

\begin{lemma}[Gate Count and Circuit Depth]
\label{lem:gatecount}
Implementing \(U_M(\beta)\) uses \(\binom{K}{2}\) logical \(XX{+}YY\) blocks per sample \(i\in S\), for a total of \(|S|\binom{K}{2} = O(K^2|S|)\) blocks. Prior to hardware transpilation, the circuit depth is \(O(K)\) per sample and remains \(O(K)\) overall, since all samples can be processed in parallel.
\end{lemma}
\begin{proof}
For each sample \(i\) and each pair \(g<h\), the unitary
\(\exp[-i\beta(X_{i,g}X_{i,h}+Y_{i,g}Y_{i,h})]\)
is one \(XX{+}YY\) block. There are \(\binom{K}{2}\) such pairs, so the total count is \(|S|\binom{K}{2}=O(K^2|S|)\).
To bound the depth for a fixed \(i\), view these pairs as edges of the complete graph \(K_K\).
When \(K\) is even, \(K_K\) decomposes into \(K-1\) perfect matchings; when \(K\) is odd, its edge-chromatic number is \(K\), yielding \(K\) matchings each acting on disjoint qubit pairs.
Thus the mixer uses \(K-1\) layers if \(K\) is even and \(K\) layers if \(K\) is odd, which is \(O(K)\).
Each \(XX{+}YY\) block further expands into a constant-size native pattern, which does not change the asymptotic depth.
Since different samples act on disjoint registers, all sample mixers can run in parallel, preserving the \(O(K)\) overall depth.
\end{proof}

\noindent\textit{Native gate accounting.}
In a native two-qubit gate set with \(RXX\) and \(RYY\), each logical \(XX{+}YY\) block maps to two native two-qubit gates per pair, for a total of \(2\binom{K}{2}\) native two-qubit gates per sample.

\begin{corollary}[Depth on Linear Connectivity]
\label{cor:depth_linear}
When the hardware graph is a linear nearest-neighbour chain, the
edges of \(K_K\) require an additional routing step; one schedule yields depth \(O(K^2)\) per sample and \(O(K^2)\) overall, still independent of \(|S|\) due to inter-sample parallelism.
\end{corollary}

\begin{proposition}[Penalty Weight Sufficiency]
\label{prop:lambda}
Let
\(
  \lambda \;>\; \sum_{i\in S}\sum_{g=1}^{K} \bigl|\Delta_{i,g}\bigr|.
\)
Then any assignment that violates at least one constraint in
Eqs.~\eqref{eqn:qubo:onehot}--\eqref{eqn:qubo:cl}
has energy strictly larger than any feasible assignment.
\end{proposition}
\begin{proof}
Each violated constraint contributes at least \(\lambda\) to the Hamiltonian in Eq.~\eqref{eqn:hamiltonian_2}, while the linear cost term can differ by at most \(\sum_{i,g}|\Delta_{i,g}|\). Therefore any infeasible assignment incurs a net energy penalty exceeding that bound.
\end{proof}

\paragraph{Penalty scale}
The lower bound on \(\lambda\) that separates feasible from infeasible assignments can conflict with upper ranges that yield a nonzero depth \(p=1\) improvement bound. The additive improvement bound in Theorem~\ref{thm:additive_p1_full} applies when \(\lambda \le \bar\eta|S|/(2\kappa)\). In practice, a schedule for \(\lambda\) can be used: moderate values during the variational search and larger values for final evaluation that enforce feasibility separation. In particular, during the search phase choose
\[
\lambda \;\le\; \lambda_{\text{limit}}
\;:=\;
\frac{\bar\eta\,|S|}{32\,(K-1)\,\bigl(|\mathcal E_{\text{ML}}|+|\mathcal E_{\text{CL}}|\bigr)},
\]
with \(\mathcal E_{\text{ML}}:=\mathcal{T}_{ml}\cap(S\times S)\) and \(\mathcal E_{\text{CL}}:=\mathcal{T}_{cl}\cap(S\times S)\), and then increase \(\lambda\) at evaluation time to enforce feasibility separation.

For the running example with \(|S|=6\) and \(K=3\), the QUBO in Eq.~\eqref{eqn:qubo} uses \(n_{\text{qubo}}=K|S|=18\) binary variables. Choosing \(\lambda>\sum_{i\in S}\sum_{g}|\Delta_{i,g}|\) ensures any constraint violation incurs a penalty above any cost improvement. The mixer in Definition~\ref{def:xy_mixer} and Lemma~\ref{lem:invariance_onehot} applies three two-qubit \(XX{+}YY\) blocks per sample, and on a linear connectivity graph the routing schedule of Corollary~\ref{cor:depth_linear} yields depth \(O(K^2)\) per sample, independent of \(|S|\).

\subsection{Additive Improvement Theorem (\(p=1\))}
\label{ssec:additive-bound}

We analyze a single-layer \(p=1\) QAOA ansatz under a global budget of \(\mathcal{O}(10^{3})\) measurement shots. In this setting, the computational cost depends on \(|S|\). Remaining improvement potential is captured via the average ambiguity margin \(\bar\eta\), which dictates the additive gain from one QAOA layer.

\begin{definition}[Incremental cost and ambiguity margin]
\label{def:margin}
Let \(g^{\star}(i)\) denote the current cluster of sample \(i\) with respect
to the present centroids \(\{\mu_g\}_{g=1}^K\).
For \(i\in S\) define the point--wise margin
\(
  \eta_i
  =
  \max\!\Bigl\{0,\;
    \min_{g\neq g^{\star}(i)}
    \bigl(
      \|x_i-\mu_{g^{\star}(i)}\|^{2}
      -\|x_i-\mu_g\|^{2}
    \bigr)
  \Bigr\}.
\)
Thus \(\eta_i>0\) if and only if at least one alternative cluster shortens
the distance of \(x_i\) to its centroid.
The average ambiguity margin is
\(
  \bar\eta=\frac{1}{|S|}\sum_{i\in S}\eta_i.
\)
\end{definition}

\paragraph{QUBO cost}
Given
\(
  \Delta_{i,g}
  =
  \|x_i-\mu_g\|^{2}
  -
  \|x_i-\mu_{g^{\star}(i)}\|^{2},
\)
let
\(
  f(z)=\sum_{i\in S}\sum_{g=1}^{K}\Delta_{i,g}\,z_{i,g}.
\)
For the feasible bit-string \(d^{(0)}\) produced by the classical pre-solver, one has \(H(d^{(0)})=f(d^{(0)})\), as all one-hot, must-link, and cannot-link constraints are satisfied, nullifying penalty terms.

\begin{theorem}[Additive improvement with one layer \(p=1\) under penalties]
\label{thm:additive_p1_full}
Let \(\ket{d^{(0)}}\in\mathcal H_{\mathrm{one\text{-}hot}}\) be a feasible state and decompose the Hamiltonian as \(H = f(\hat d) + H_{\text{pen}}\), where \(H_{\text{pen}} = H_{\text{one-hot}} + H_{\text{ML}} + H_{\text{CL}}\).
Let \(B=\sum_{i\in S}\sum_{g<h}(X_{i,g}X_{i,h}+Y_{i,g}Y_{i,h})\) be the one-hot preserving mixer from Lemma~\ref{lem:invariance_onehot}.
Choose \(\beta = c/\sqrt{|S|}\) with \(0<c\le\pi/6\), and define \(\ket{\psi_1}=U_M(\beta)\,U_C(\gamma)\ket{d^{(0)}}\), with \(U_M(\beta)=e^{-i\beta B}\) and \(U_C(\gamma)=e^{-i\gamma H}\).
Then
\[
\langle\psi_1|H|\psi_1\rangle
=
f(d^{(0)})
-
(\bar\eta\,|S|-\kappa\lambda)\,\beta^{2}
+
\mathcal O(\beta^{3}),
\]
where \(\kappa := 16\,(K-1)\bigl(|\mathcal E_{\text{ML}}|+|\mathcal E_{\text{CL}}|\bigr)\), with \(\mathcal E_{\text{ML}}:=\mathcal{T}_{ml}\cap(S\times S)\) and \(\mathcal E_{\text{CL}}:=\mathcal{T}_{cl}\cap(S\times S)\).
Moreover, if \(\lambda \le \frac{\bar\eta\,|S|}{2\kappa}\), then
\[
\langle\psi_1|H|\psi_1\rangle
\le
f(d^{(0)})
-
\frac{1}{2}\,c^{2}\,\bar\eta
+
\mathcal O\!\bigl(|S|^{-1/2}\bigr),
\]
which yields an additive improvement proportional to \(\bar\eta\).
\end{theorem}
\begin{proof}
Since \(U_C(\gamma)=e^{-i\gamma H}\) is diagonal and \(\ket{d^{(0)}}\) is an eigenstate of \(H\), it contributes only a global phase, so
\(\langle\psi_1|H|\psi_1\rangle
= \bra{d^{(0)}}U_M^{\dagger}(\beta)\,H\,U_M(\beta)\ket{d^{(0)}}\).
The one-hot penalty \(H_{\text{one-hot}}\) annihilates \(\ket{d^{(0)}}\) and \(U_M(\beta)\) preserves \(\mathcal H_{\mathrm{one\text{-}hot}}\), so its contribution vanishes.
Using the Baker--Campbell--Hausdorff expansion,
\(U_M^{\dagger} H U_M
= H + i\beta[B,H]
+ \tfrac{\beta^{2}}{2}[B,[B,H]]
+ \mathcal O(\beta^{3})\).
Because \(H\) is diagonal and \(B\) flips two bits, \(\bra{d^{(0)}}[B,H]\ket{d^{(0)}}=0\).
Taking expectations,
\[
\langle\psi_1|H|\psi_1\rangle
= f(d^{(0)})
+ \frac{\beta^{2}}{2}\bra{d^{(0)}}[B,[B,f]]\ket{d^{(0)}}
+ \frac{\beta^{2}}{2}\bra{d^{(0)}}[B,[B,H_{\text{pen}}]]\ket{d^{(0)}}
+ \mathcal O(\beta^{3}).
\]
From the analysis of the linear term, \([B,[B,f]] = -2\sum_{i\in S}\eta_i\), giving the term \(-\bar\eta\,|S|\,\beta^{2}\).
For each must-link or cannot-link edge, a nested-commutator norm bound \(\|[B,[B,\cdot]]\|\le 16\,(K-1)\) holds; summing over \(|\mathcal E_{\text{ML}}|+|\mathcal E_{\text{CL}}|\) edges yields
\(\tfrac{\beta^{2}}{2}\bra{d^{(0)}}[B,[B,H_{\text{pen}}]]\ket{d^{(0)}} \le \kappa\lambda\beta^{2}\).
Combining these bounds gives the claim.
If \(\lambda \le \bar\eta\,|S|/(2\kappa)\) the coefficient of \(\beta^{2}\) is at most \(-\tfrac12\bar\eta\); with \(\beta=c/\sqrt{|S|}\) the remainder scales as \(\mathcal O(|S|^{-1/2})\), completing the proof.
\end{proof}

\subsection{Comparison with Classical Heuristic Algorithms}
\label{sec:b_heuristic}

Two variants of \emph{q-PCKMeans} are compared against three widely used constrained-clustering heuristics: COP-KMeans \citep{wagstaff_constrained_2001}, MPCKMeans \citep{bilenko2004integrating}, and PCKMeans \citep{basu_active_2004}. Across datasets, penalty-based heuristics such as PCKMeans and MPCKMeans can reduce SSE on some instances but may incur nonzero violation rates, whereas COP-KMeans enforces constraints when it returns a solution but may fail to return in some regimes. q-PCKMeans attains low violation rates and remains competitive in SSE; when the restricted search does not reach feasibility, residual violations are reported under the evaluation protocol.

\begin{table*}[htbp]
\centering
\caption{Performance of classical clustering heuristics with constraints.}
\label{tab:heuristics}
\begin{adjustbox}{max width=\textwidth, max totalheight=\textheight-2\baselineskip, keepaspectratio}
\begin{tabular}{llrrrrrrrrr}
\toprule
\multicolumn{2}{c}{}
  & \multicolumn{3}{c}{Cannot-link}
  & \multicolumn{3}{c}{Must-link}
  & \multicolumn{3}{c}{Both} \\
\cmidrule(lr){3-5} \cmidrule(lr){6-8} \cmidrule(lr){9-11}
Dataset & Method
  & SSE & \makecell{Runtime\\(s)} & \makecell{Violation\\Rate}
  & SSE & \makecell{Runtime\\(s)} & \makecell{Violation\\Rate}
  & SSE & \makecell{Runtime\\(s)} & \makecell{Violation\\Rate} \\
\midrule
Seeds     & cop-kmeans               & 887.30   & 1.2976 & 0.0000 & \multicolumn{3}{c}{No solution found}      & \multicolumn{3}{c}{No solution found}      \\
n = 210   & mpckmeans                & 705.93   & 0.0823 & 0.0000 & 1092.12  & 0.0689 & 0.0152 & 888.79   & 0.2333 & 0.0010 \\
m = 7     & pckmeans                 & 855.49   & 0.0479 & 0.0000 & 1011.51  & 0.0340 & 0.0162 & 964.02   & 0.0355 & 0.0163 \\
\(K\)=3   & q-PCKMeans (CA)          & 682.73& 43.2233& 0.0000 & 691.74& 38.2600& 0.0000 & 724.19& 45.1886& 0.0000 \\
          & q-PCKMeans (IG)          & 682.73& 62.6145& 0.0000 & 733.93& 65.2015& 0.0000 & 724.19& 62.0752& 0.0000 \\
\midrule
Haberman  & cop-kmeans               & \multicolumn{3}{c}{No solution found}      & \multicolumn{3}{c}{No solution found}      & \multicolumn{3}{c}{No solution found}      \\
n = 306   & mpckmeans                &5.18E+04  & 0.1279 & 0.2288 &5.15E+04  & 0.0969 & 0.0588 &5.02E+04  & 0.6596 & 0.0539 \\
d = 3     & pckmeans                 &5.04E+04  & 0.0593 & 0.0510 &5.07E+04  & 0.0425 & 0.0725 &5.04E+04  & 0.0400 & 0.0605 \\
\(K\)=2   & q-PCKMeans (CA)          &5.18E+04& 48.3982& 0.0000 &5.25E+04& 50.8746& 0.2939 &5.32E+04& 106.8068& 0.0000 \\
          & q-PCKMeans (IG)          &5.20E+04& 52.5752& 0.0000 &5.26E+04& 58.1897& 0.0088 &5.32E+04& 142.3847& 0.0000 \\
\midrule
Raisins   & cop-kmeans               & \multicolumn{3}{c}{No solution found}      & \multicolumn{3}{c}{No solution found}      & \multicolumn{3}{c}{No solution found}      \\
n = 900   & mpckmeans                &1.55E+12  & 0.8031 & 0.0742 &1.99E+12  & 0.6772 & 0.0289 &1.88E+12  & 8.6050 & 0.0167 \\
m = 7     & pckmeans                 &1.78E+12  & 0.4880 & 0.0371 &1.76E+12  & 0.4460 & 0.0291 &1.86E+12  & 0.5017 & 0.0293 \\
\(K\)=3   & q-PCKMeans (CA)          &1.35E+12& 45.8665& 0.0000 &1.37E+12& 77.6761& 0.1481 &1.71E+12&102.1515& 0.0000 \\
          & q-PCKMeans (IG)          &1.35E+12& 41.7418& 0.0000 &1.40E+12& 75.7384& 0.0000 &1.61E+12&104.7233& 0.0000 \\
\bottomrule
\end{tabular}
\end{adjustbox}
\end{table*}

To our knowledge, there are no published quantum methods that directly address pairwise-constrained \(k\)-means; existing NISQ studies focus on unconstrained \(k\)-means. The comparison therefore uses standard classical heuristics under identical inputs and constraint regimes. Exact ILP solvers and specialized meta-heuristics are omitted because they target different formulations and entail different computational assumptions relative to the simulation-based NISQ setting considered here.

\subsection{Clustering evaluation}
\label{sec:eval}
This section reports external clustering metrics (ARI, AMI, and purity) on six datasets with ground-truth labels. Following Section~\ref{sec:pairwise_aware}, solution quality is assessed by the \(k\)-means objective (SSE), while ARI, AMI, and purity provide an additional view of label agreement. The number of clusters matches the ground-truth label count.

\begin{table}[H]
  \caption{Clustering evaluation metrics on solutions of datasets under different constraint settings.}
  \label{tab:clustering_metrics}
  \vskip 0.1in
  \centering
  \renewcommand\cellalign{lc}
  \begin{small}
  \begin{sc}
  \setlength{\tabcolsep}{6pt}
  \resizebox{0.9\textwidth}{!}{%
    \begin{tabular}{llcccccc}
      \toprule
      \multirow[c]{1}{*}{Metrics} & \multirow[c]{1}{*}{Constraints} & \makecell[c]{Iris\\$K = 3$} & \makecell[c]{Seeds\\$K = 3$} & \makecell[c]{Haberman\\$K = 2$} & \makecell[c]{Moons\\$K = 2$} & \makecell[c]{Spiral\\$K = 2$} & \makecell[c]{Land\_mine\\$K = 5$} \\
      \midrule
      \multirow{4}{*}{ARI}
        & MSSC & 0.5912 & 0.6890 & 0.0015 & 0.4977 & 0.0005 & 0.0782 \\
        & q-PCKMeans$_{\mathrm{W}}$ (ML) & 0.6051 & 0.7208 & 0.0070 & 0.5864 & 0.0054 & 0.0986 \\
        & q-PCKMeans$_{\mathrm{W}}$ (CL) & 0.5941 & 0.6890 & 0.0097 & 0.5362 & 0.0017 & 0.0790 \\
        & q-PCKMeans$_{\mathrm{W}}$ (ML+CL) & 0.6051 & 0.7116 & 0.0058 & 0.5560 & 0.0054 & 0.0905 \\
      \midrule
      \multirow{4}{*}{AMI}
        & MSSC & 0.6302 & 0.6567 & 0.0006 & 0.3971 & 0.0004 & 0.1216 \\
        & q-PCKMeans$_{\mathrm{W}}$ (ML) & 0.6547 & 0.6857 & 0.0060 & 0.4791 & 0.0039 & 0.1402 \\
        & q-PCKMeans$_{\mathrm{W}}$ (CL) & 0.6428 & 0.6567 & 0.0130 & 0.4321 & 0.0013 & 0.1219 \\
        & q-PCKMeans$_{\mathrm{W}}$ (ML+CL) & 0.6547 & 0.6860 & 0.0043 & 0.4506 & 0.0039 & 0.1306 \\
      \midrule
      \multirow{4}{*}{Purity}
        & MSSC & 0.8067 & 0.8048 & 0.7353 & 0.7500 & 0.5133 & 0.3017 \\
        & q-PCKMeans$_{\mathrm{W}}$ (ML) & 0.8200 & 0.8952 & 0.7353 & 0.8833 & 0.5467 & 0.3550 \\
        & q-PCKMeans$_{\mathrm{W}}$ (CL) & 0.8133 & 0.8810 & 0.7353 & 0.8667 & 0.5200 & 0.3462 \\
        & q-PCKMeans$_{\mathrm{W}}$ (ML+CL) & 0.8200 & 0.8905 & 0.7353 & 0.8733 & 0.5467 & 0.3462 \\
      \bottomrule
    \end{tabular}%
  }
  \end{sc}
  \end{small}
\end{table}

\begin{table*}[ht]
  \centering
  \caption{SSE for different backends and shot counts from 256 to 2048.}
  \label{tab:sse_noise_experiments}
  \begin{sc}
  \resizebox{\textwidth}{!}{%
  \begin{tabular}{c|l|cccc|cccc|cccc}
    \toprule
    \multirow{2}{*}{Constraint}& \multirow{2}{*}{Backend}& \multicolumn{4}{c|}{Iris}
      & \multicolumn{4}{c|}{Seeds}
      & \multicolumn{4}{c}{Monk\_2} \\
    \cmidrule{3-14}
      &
      & 256 & 512 & 1024 & 2048
      & 256 & 512 & 1024 & 2048
      & 256 & 512 & 1024 & 2048 \\
    \midrule
      \multirow{4}{*}{CL} & Ideal       & 131.21 & 137.46 & 138.29 & 101.16& 861.91
& 855.96
& 853.45
& 682.73& 1547.44
& 1543.70
& 1547.80
& 1553.84\\
      & FLimaV2     & 156.64 & 156.31 & 131.88 & 131.57 & 883.29
& 882.59
& 851.02
& 847.75
& 1548.29
& 1545.62
& 1547.15
& 1545.67
\\
      & FJakartaV2  & 143.94 & 145.99 & 132.25 & 133.54 & 861.50
& 832.84
& 860.74
& 802.59
& 1547.84
& 1548.82
& 1545.71
& 1548.30
\\
      & FBrooklynV2 & 147.05 & 142.92 & 142.44 & 127.66 & 865.11& 841.89& 811.02& 862.77& 1549.52& 1549.39& 1547.02& 1547.87\\
    \midrule
      \multirow{4}{*}{ML} & Ideal       & 176.31 & 190.65 & 168.43 & 106.09& 969.90
& 920.77
& 900.90
& 733.93& 1547.73 & 1548.07 & 1548.06 & 1551.67\\
      & FLimaV2     & 178.23 & 168.47 & 191.07 & 166.91 & 973.36
& 926.03
& 918.26
& 891.71
& 1548.90 & 1549.37 & 1546.63 & 1548.98 \\
      & FJakartaV2  & 190.85 & 168.79 & 170.90 & 190.57 & 940.56
& 976.28
& 915.39
& 937.48
& 1548.63 & 1549.19 & 1548.73 & 1548.47 \\
      & FBrooklynV2 & 184.96 & 190.02 & 172.68 & 178.03 & 896.85& 910.17& 951.63& 953.02& 1546.19 & 1550.41 & 1549.18 & 1545.93 \\
    \midrule
      \multirow{4}{*}{Both} & Ideal       & 153.51 & 153.09 & 161.60 & 97.62& 809.92
& 838.85
& 814.06
& 724.19& 1543.12 & 1545.29 & 1542.75 & 1581.74\\
      & FLimaV2     & 160.46 & 176.15 & 160.32 & 141.74 & 849.84
& 848.48
& 831.94
& 844.48
& 1543.00 & 1545.00 & 1546.20 & 1549.23 \\
      & FJakartaV2  & 160.90 & 161.99 & 148.97 & 150.64 & 830.13
& 887.40
& 837.10
& 807.61
& 1545.82 & 1545.38 & 1544.11 & 1546.30 \\
      & FBrooklynV2 & 170.90 & 174.86 & 148.71 & 153.15 & 858.02& 872.79& 831.61& 826.52& 1547.33 & 1547.11 & 1547.47 & 1550.87 \\
    \bottomrule
  \end{tabular}
  }
  \end{sc}
\end{table*}

\subsection{Robustness to Noise}
\label{ssec:robustness}

Noise sensitivity is evaluated using AerSimulator with one ideal simulator and three IBM-Q noise models (FakeLimaV2, FakeJakartaV2, FakeBrooklynV2), varying the shot budget from 256 to 2048. Table~\ref{tab:sse_noise_experiments} summarizes SSE under each configuration. Increasing the shot budget generally reduces SSE across the tested noise models, with diminishing returns by 2048 shots. Minor non-monotonic variations occur for specific backend and dataset pairs.

\section{Penalty Weight \(\lambda\)}

Across both easy (Block~A) and hard (Block~B) instances, all three constraint types (\texttt{ml}, \texttt{cl} and \texttt{both}) and five random seeds, a plateau is observed over \(0.05 \le \lambda \le 5\): SSE deviates by less than 3\,\% from the reference value, and \(\lambda=50\) worsens SSE by about 1.5\,\%.
The number of ML and CL violations decreases as \(\lambda\) grows, while wall-clock time and logical-qubit usage vary by at most 5\,\% (Fig.~\ref{fig:runtime_qubits_vs_lambda}).
This plateau aligns with the additive-improvement window predicted by Theorem~\ref{thm:additive_p1_full},
\[
\lambda \;\lesssim\; \frac{\bar\eta\,|S|}{2\kappa}.
\]
Boxplots in Fig.~\ref{fig:sse_distribution_vs_lambda} show overlapping SSE distributions across \(\lambda\in\{0.05,0.1,0.5,1,2,5,10,20,50\}\).
Within the tested datasets and constraint densities, tuning \(\lambda\) in the range 0.05--5 is not required for stable runtime and low violation counts.

\begin{figure*}[htbp]
    \centering
    \begin{minipage}{0.48\linewidth}
        \centering
        \includegraphics[width=\linewidth]{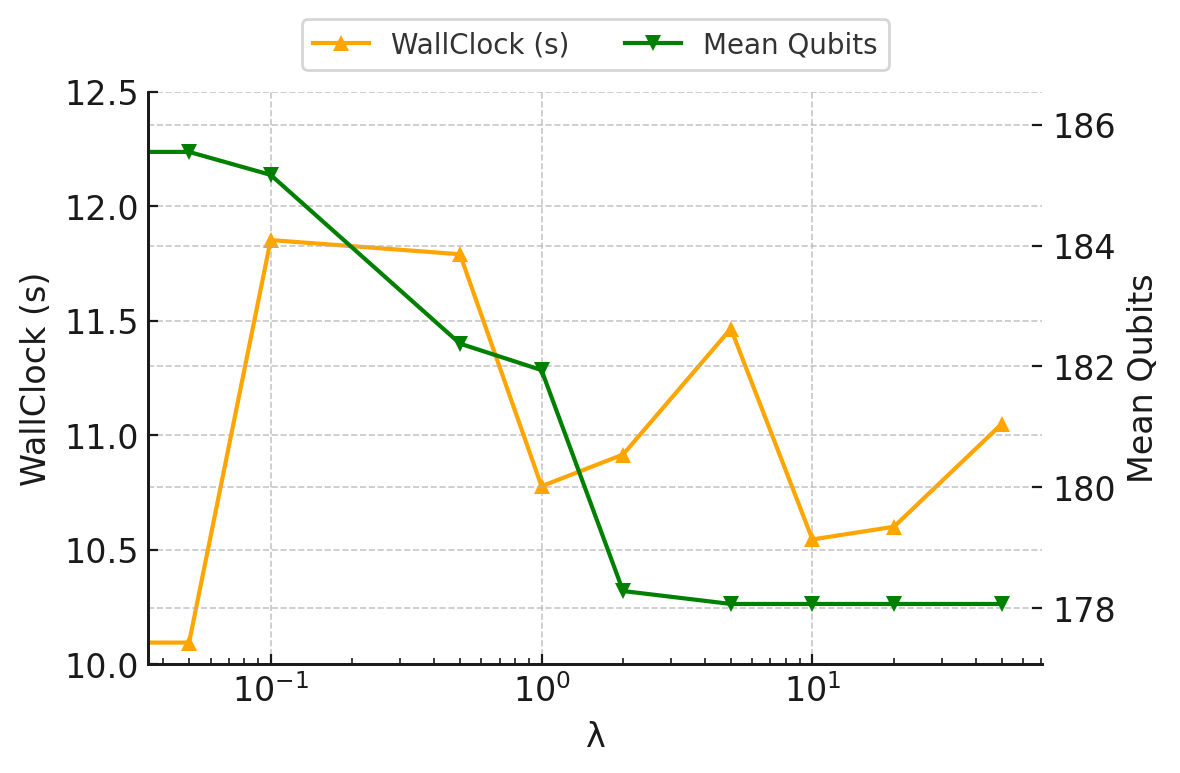}
        \caption{Wall-clock time and logical-qubit usage versus penalty weight \(\lambda\) (log scale). Both metrics vary by less than 5\,\% across the sweep.}
        \label{fig:runtime_qubits_vs_lambda}
    \end{minipage}\hfill
    \begin{minipage}{0.48\linewidth}
        \centering
        \includegraphics[width=\linewidth]{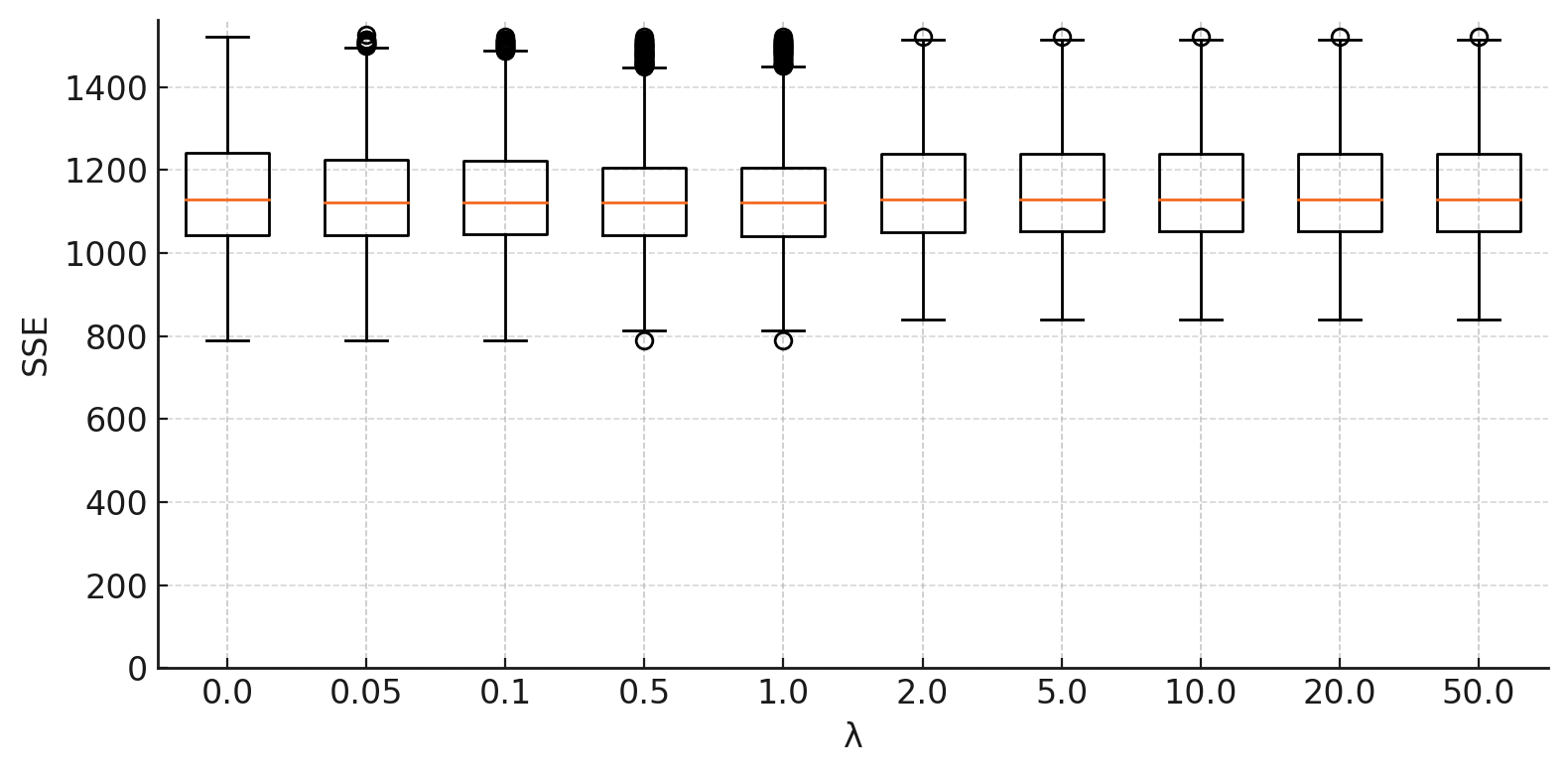}
        \caption{Boxplots of SSE distributions for each tested \(\lambda\) value. Medians and interquartile ranges overlap across \(\lambda\in\{0.05,0.1,0.5,1,2,5,10,20,50\}\).}
        \label{fig:sse_distribution_vs_lambda}
    \end{minipage}
\end{figure*}

\subsection{Real Hardware Implementation}
Real hardware runs are omitted because current NISQ devices remain noisy, which limits reliability for clustering experiments \citep{bharti2022noisy}. Median two-qubit error rates remain about 0.6\%, so a depth-1 QAOA circuit with at least $10^3$ entangling gates would succeed with probability below $10^{-3}$ per shot, making statistical comparisons impractical without many repetitions \citep{xue2021effects}. Current devices also limit the dataset sizes that can be encoded. Hardware noise is therefore emulated with calibrated AerSimulator models.

\subsection{Societal Impact of q-PCKMeans}

q-PCKMeans may support data analysis in domains such as genomics, medicine, and materials science. As with classical \(k\)-means, clustering may reproduce patterns present in the input data, including group disparities. Fairness-aware clustering variants, including group-balanced objectives \citep{backurs2019scalable} and demographic constraints \citep{chierichetti2017fair}, are possible directions for mitigation. Clustering sensitive data can also introduce privacy risks, including re-identification through linkage attacks. Differential privacy methods, such as adding calibrated noise to centroid updates \citep{dwork2006calibrating}, provide another direction. The effectiveness of fairness and privacy techniques in quantum-assisted clustering remains an open question.


\section{Extended related works}

\subsection{Quantum Computing Concepts}

Classically, the fundamental unit of digital information is the \emph{bit}, definitively holding a value of either 0 or 1. Quantum computation generalizes this notion through the \emph{quantum bit}, or \emph{qubit}, which can exist simultaneously in a coherent superposition of both basis states \citep{nielsen2010quantum,deutsch1985quantum}.  We work in the Noisy Intermediate-Scale Quantum (NISQ) era, characterized by devices with tens to a few hundred qubits that are subject to noise and decoherence; such machines cannot yet support fully fault-tolerant algorithms but may still offer quantum advantage for certain tasks \citep{preskill2018quantum,arute2019quantum}.

\paragraph{States and Control Primitives}
\label{subsec:states_control}
Formally, a single qubit lives in the two-dimensional Hilbert space  
\(
  \mathcal{H}_{2}=\mathrm{span}\{\ket{0},\ket{1}\},
\)
where \(\ket{0}\) and \(\ket{1}\) are the Pauli-\(Z\) eigenstates.  Any pure state can be written as
\[
  \ket{\psi}=\alpha\ket{0}+\beta\ket{1},\quad
  |\alpha|^2+|\beta|^2=1.
\]
Geometrically, \(\ket{\psi}\) corresponds to a point on the Bloch sphere with polar angles \((\theta,\phi)\), where
\(\alpha=\cos(\tfrac{\theta}{2})\) and \(\beta=e^{i\phi}\sin(\tfrac{\theta}{2})\).
Universal gate-based control uses single-qubit rotations 
\[
  R_{\alpha}(\theta)=\exp\!\bigl(-i\tfrac{\theta}{2}\,\sigma_{\alpha}\bigr),
  \quad \alpha\in\{x,y,z\},
\]
and two-qubit entanglers (\textsc{CNOT}) \citep{barenco1995elementary}.  
Here \(\sigma_{\alpha}\) are the Pauli matrices.  
The \emph{circuit depth} or the number of sequential gate layers serves as a proxy for runtime and accumulated noise on NISQ hardware. In our experiments, we emulate such low-depth circuits with a matrix-product-state (MPS) backend, which compactly represents the limited entanglement produced by shallow QAOA layers.

\paragraph{QUBO to Ising Mapping}

Many combinatorial problems can be expressed as a Quadratic Unconstrained Binary Optimization (QUBO) problem
\[
  \min_{z\in\{0,1\}^N}\; z^\top Q\,z \;+\; c^\top z,
\]
where \(Q\in\mathbb R^{N\times N}\) and \(c\in\mathbb R^N\).  To leverage quantum hardware, we map the binary variables \(z_i\in\{0,1\}\) to Ising spins \(s_i\in\{\pm1\}\) via
\(
  s_i = 2\,z_i - 1.
\)
Under this change of variables the QUBO objective becomes, up to an additive constant,
\(
  H_C \;=\;\sum_{i<j} J_{ij}\,Z_iZ_j \;+\;\sum_i h_i\,Z_i,
\)
where \(Z_i\) is the Pauli-\(Z\) operator on qubit \(i\), and the couplings \(J_{ij}\) and fields \(h_i\) are algebraically determined by the entries of \(Q\) and \(c\).  The resulting Hamiltonian \(H_C\) encodes the original cost landscape as its diagonal energies, ready for optimization by QAOA.

\paragraph{Quantum Approximate Optimization Algorithm}

The Quantum Approximate Optimization Algorithm (QAOA) is a hybrid quantum--classical method for approximately solving combinatorial optimization problems by finding the ground state of a diagonal cost Hamiltonian \(H_C\) on \(N\) qubits (Hilbert‑space dimension \(2^N\); here \(N = |S|\times K\)):
\[
  H_C = \sum_{z\in\{\pm 1\}^N} C(z)\,\ket{z}\!\bra{z}\,,
\]
where \(C(z)\) is the classical objective, such as an Ising-encoded QUBO, and \(\{\ket{z}\}\) denotes the computational basis. QAOA constructs a variational state by alternating between the cost and mixer unitaries
\(
  U_C(\gamma) = e^{-i\gamma H_C}, 
  \quad
  U_M(\beta) = e^{-i\beta H_M},
\)
where \(H_M\) is a mixer Hamiltonian that induces transitions between computational‑basis states.  At depth \(P\), the ansatz is
{\small\[\ket{\bm\gamma,\bm\beta}=\Bigl(U_M(\beta_P)\,U_C(\gamma_P)\Bigr)\,
    \cdots\,
    \Bigl(U_M(\beta_1)\,U_C(\gamma_1)\Bigr)\,
    \ket{+}^{\otimes N},
\]}
with \(2P\) variational parameters \(\{\gamma_1,\dots,\gamma_P;\,\beta_1,\dots,\beta_P\}\). A common choice for the mixer is the XY Hamiltonian
\(
  H_M = \sum_{i<j}\bigl(X_iX_j + Y_iY_j\bigr),
\)
which preserves Hamming-weight subspaces, thereby enforcing feasibility constraints. The QAOA loop proceeds by (i) preparing \(\ket{\bm\gamma,\bm\beta}\), (ii) measuring \(\langle H_C\rangle\) on repeated circuit runs, and (iii) updating \(\bm\gamma,\bm\beta\) via a classical optimizer such as COBYLA to maximise the expected cost.  Although theory guarantees monotonic improvement as \(P\to\infty\), in practice small \(P\) is chosen on NISQ devices to trade off circuit complexity against approximation quality.

\subsection{Extension of constrained clustering with NISQ-friendly QUBO formulation}
Quantum computing has attracted interest for clustering, with the goal of
obtaining quantum advantage or practical gains in hybrid pipelines
\cite{saiphet2021quantum,yung2024clustering,chen2025provably}. Early work casts
clustering objectives as Quadratic Unconstrained Binary Optimization (QUBO)
problems and solves them on quantum annealers \cite{zaiou2021balanced} or with
gate based QAOA, establishing feasibility \citep{boros2007local}. These
approaches are limited by nonconvex energy landscapes, which benefit from
careful warm starts or mixer design, and by the quadratic growth of one hot QUBO
encodings, which strains qubit counts and coherence times
\cite{kumari2018quantum,li2020coreset,egger2021warmstart,mirkarimi2024isingpenalty}.
To improve scaling, some methods compress the dataset before forming QUBOs; for
instance, qc-kmeans builds a constant size sketch and then solves small QUBOs
with shallow QAOA \citep{chumpitaz2025qc}. Empirical studies therefore remain
focused on small datasets or reduced coresets, with accuracy that often matches
or trails classical baselines. Despite progress in quantum clustering without
constraints, only a few annealer based proofs of concept address constrained
clustering \cite{cohen2020ising,richoux2023learning,seong2025hamiltonian}. These
works do not provide a systematic treatment of pairwise constraints, leaving
this setting largely open.

\end{document}